\documentclass[11pt]{article}

\usepackage[total={6.5in, 9.5in}]{geometry}

\usepackage[T1]{fontenc} 
\usepackage[utf8]{inputenc}

\usepackage[labelfont=bf]{caption}
\usepackage[format=hang]{subcaption}

\usepackage{hyperref}
\usepackage{url}
\usepackage{wrapfig}

\usepackage{amsmath}
\usepackage{amssymb}
\usepackage{amsthm}
\usepackage{xfrac}

\usepackage{graphicx}
\usepackage[table]{xcolor}
\usepackage{tikz}
\usetikzlibrary{bayesnet}

\usepackage[labelfont=bf]{caption}
\usepackage[format=hang]{subcaption}

\usepackage{booktabs}
\usepackage{multirow}
\usepackage{arydshln}
\usepackage{makecell}

\usepackage{algorithm}
\usepackage{algpseudocode}

\usepackage{enumitem}

\usepackage{siunitx}
\usepackage{textcomp}
\usepackage{xurl}
\usepackage[acronym]{glossaries}

\usepackage{cleveref}
\usepackage{natbib}

\crefname{section}{sec.}{secs.}
\Crefname{section}{Sec.}{Secs.}

\crefname{figure}{fig.}{figs.}
\Crefname{figure}{Fig.}{Figs.}

\crefname{table}{tab.}{tabs.}
\Crefname{table}{Tab.}{Tabs.}

\crefname{equation}{eq.}{eqs.}
\Crefname{equation}{Eq.}{Eqs.}

\crefname{algorithm}{alg.}{algs.}
\Crefname{algorithm}{Alg.}{Algs.}

\crefname{appendix}{app.}{apps.}
\Crefname{appendix}{App.}{Apps.}

\newtheorem{proposition}{Proposition}

\newtheorem{example}{Example}
\newenvironment{proofsketch}{\proof
}{\endproof
}

\newcommand{\f}{\mathrm{f}}

\newcommand{\g}{\mathrm{g}}

\newcommand{\pr}{\mathrm{p}}

\newcommand{\Prob}{\mathbb{P}}

\newcommand{\di}{\mathrm{d}}

\newcommand{\Xc}{\mathcal{X}}
\newcommand{\Fc}{\mathcal{F}}
\newcommand{\Yc}{\mathcal{Y}}
\newcommand{\Hc}{\mathcal{H}}

\newcommand{\Ib}{\mathbb{I}}
\newcommand{\Nb}{\mathbb{N}}

\newcommand{\Rb}{\mathbb{R}}

\newcommand{\Eb}{\mathbb{E}}
\newcommand{\Bc}{\mathcal{B}}

\newcommand{\Ic}{\mathcal{I}}
\newcommand{\Sc}{\mathcal{S}}
\newcommand{\Uc}{\mathcal{U}}
\newcommand{\Nc}{\mathcal{N}}
\newcommand{\Mc}{\mathcal{M}}

\newcommand{\Vc}{\mathcal{V}}

\newcommand{\Dc}{\mathcal{D}}
\newcommand{\Pc}{\mathcal{P}}

\newcommand{\Zc}{\mathcal{Z}}

\newcommand{\Oc}{\mathcal{O}}
\newcommand{\Qc}{\mathcal{Q}}
\newcommand{\cat}{\mathrm{Cat}}

\newcommand{\amax}{\mathrm{argmax}}
\newcommand{\unif}{\mathrm{Uniform}}

\newcommand{\EIG}{\textsc{EIG}}

\newcommand{\KL}{\textsc{KL}}
\newcommand{\MI}{\textsc{MI}}

\title{Information-Dense Synthesis for Molecular Discovery}

\author{Kasper K. Jakobsen$^*$ \& Eli N. Weinstein\footnote{Department of Chemistry, Technical University of Denmark, Kgs. Lyngby, DK. Contact: \url{enawe@dtu.dk}}}

\begin{document}

\maketitle

\begin{abstract}
Machine learning can accelerate molecular discovery by designing molecules and planning experiments.
However, many scientific challenges demand molecules with very rare properties, and in this sparse setting, existing algorithms offer little gain over random guessing.
We propose a method to efficiently search large regions of molecular space using algorithmically controlled stochastic synthesis. 
Rather than design, make and test individual molecules, we design and make complex mixtures, test them as a pool, then deconvolute the molecule-activity map. 
We optimize synthesis to encode maximal information.
Theoretically, this approach can reduce the number of experiments required to find the optimal molecule among $d$ candidates from $\mathcal{O}(d)$ to $\mathcal{O}(\log d)$ or $\mathcal{O}(1)$.
In simulation, on estimated protein fitness landscapes, it finds active molecules with an order of magnitude fewer experiments than existing Bayesian optimization methods.
\end{abstract}

\section{Introduction}

Lab-in-the-loop ML and AI systems promise to efficiently explore molecular space, by iteratively learning from data and generating new molecules to test \citep[e.g.][]{romero_exploring_2009, romero_navigating_2013, frey_lab---loop_2025, rapp_self-driving_2024, liu_amorphous_2025, mints_exploring_2022, hastrup_atomcomposer_2026}.
But very often, the region of molecular space where molecules have a desired property is exceedingly narrow:
within the combinatorial vastness of possible molecules, only a small subset catalyzes a reaction or treats a disease.
This \textit{sparseness} poses fundamental challenges to existing lab-in-the-loop systems, which
design and test candidate molecules individually. Most candidates fail, and therefore provide almost no information about the optimal molecule.
On sparse problems, even state-of-the-art algorithms perform no better than random (\Cref{thm:indiv-limit}) \citep{Bach2020-yw}.
\begin{wrapfigure}{r}{0.4\textwidth}
    \centering
    \includegraphics[width=0.4\textwidth]{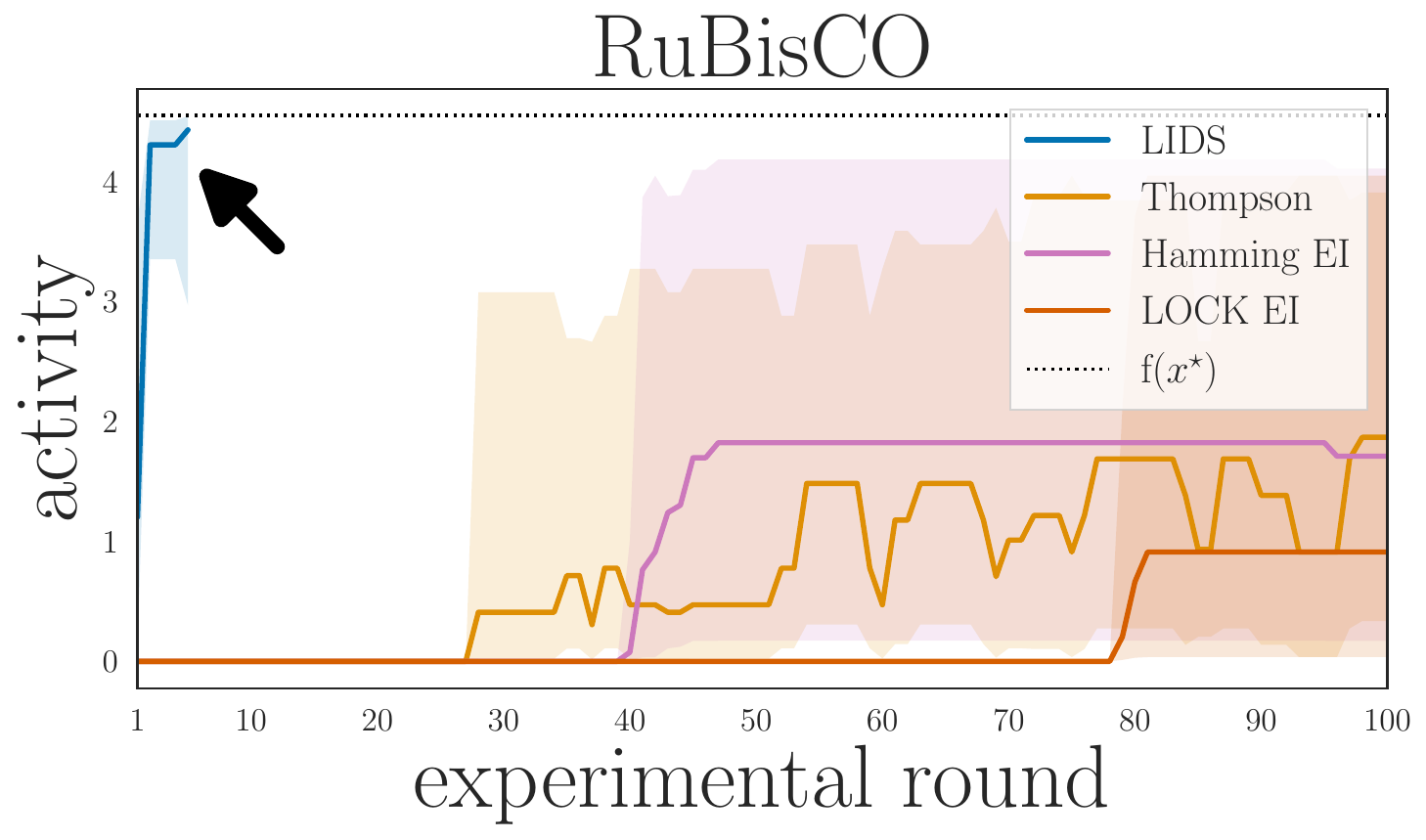}
    \caption{\textbf{Lab-in-the-loop speedup.} Optimizing an \textit{in silico} oracle with LIDS (blue) versus baselines (\Cref{sec:results}).
    }\label{fig:preview}
\end{wrapfigure}
\noindent
In this article we propose lab-in-the-loop learning with \textit{information-dense} synthesis (LIDS), a method that designs, builds and tests complex mixtures of molecules, rather than individual molecules. 
It is tractable to synthesize complex mixtures, without added laboratory time or cost, by combinatorial chemistry \citep{Szostak1997-wh,houghten_mixture-based_1999}. 
Our idea is to control these stochastic chemical reactions algorithmically, within the framework of Bayesian optimization and experimental design \citep{rainforth_modern_2024,shahriari_taking_2016}. 
We use generative models to describe stochastic synthesis \citep{weinstein_optimal_2022, weinstein_manufacturing-aware_2026}. 
We test the synthesized mixture as a pool \citep{liu_scalable_2025}. We optimize synthesis so the experiment provides maximum information about the molecule with the highest activity \citep{hernandez-lobato_predictive_2014,russo_learning_2018}. The method yields rapid acceleration on sparse molecular discovery problems (\Cref{fig:preview}).

\noindent
Our contributions are: 
\begin{enumerate}
    \item A recipe to accelerate lab-in-the-loop molecular discovery by controlled stochastic synthesis, rather than deterministic synthesis. 
    \item Efficient approximate Bayesian computation for designing stochastic syntheses, and learning from assays that test the resulting mixtures as a pool (\Cref{sec:method}). 
    \item Theoretical results showing exponential speedups over methods that design and test molecules individually (\Cref{sec:theory}). 
    \item Empirical demonstration of order-of-magnitude speedups across diverse simulated DNA and peptide activity landscapes (\Cref{sec:results}).
\end{enumerate}
Notation is defined where introduced; a table is in \Cref{app:notation}.

\section{Problem Setting}

Our goal is to discover molecules with a desired property. Let $\Xc$ denote the space of molecules we are exploring, and let $\f: \Xc \to \Rb$ denote the mapping from molecule to activity. Our goal is to find $x^\star \in \amax_x \f(x)$.
For example, $\Xc$ could be the space of peptide sequences of length $L$, i.e. $\Xc = \{A, C, \ldots\}^L$, and $\f(x)$ could be activation of a target receptor, e.g. GLP-1. 
We study lab-in-the-loop optimization methods where we query $\f(x)$ by running an experiment.

\paragraph{Current approach.} The usual technique for solving this problem is to iteratively (1) design a molecule $x_t$, (2) synthesize $x_t$ in the laboratory, (3) test $x_t$ in an assay that returns $y_t$, a noisy measurement of $\f(x_t)$, and (4) update our knowledge of $\f$ and its optimum.
Algorithms based on Bayesian optimization or reinforcement learning can be used to learn from the data $\Dc_{t-1} = \{(x_1, y_1), \ldots, (x_{t-1}, y_{t-1})\}$ and propose new molecules $x_t$ (steps 1 and 4) \citep{romero_exploring_2009,romero_navigating_2013, maus_purely_2026,groth_kermut_2024,jankowiak_flexible_2026}. Robotic systems are increasingly used to automate synthesis and testing (steps 2 and 3)  \citep{ha_ai-driven_2023}. Compiled together, steps 1-4 have become a basic recipe for automated or ``self-driving" labs in the molecular sciences \citep{rapp_self-driving_2024,abolhasani_lab_2026}. 

\paragraph{Core challenge.}   
Active molecules often occupy a narrow region of the molecular space. For example: antibody screens turn up one useful sequence among a billion candidates \citep{skora_generation_2015}; flipping a single chiral center in thalidomide makes the drug toxic \citep{vargesson_thalidomide-induced_2015}; 0.14\% of single point mutations in RuBisCO, a $\mathrm{CO}_2$ capturing enzyme, increase activity over baseline \citep{prywes_map_2025}.
The sparseness of molecule-activity maps poses a fundamental challenge to existing lab-in-the-loop discovery methods.
To illustrate, consider a simple setting in which we have $d \triangleq |\Xc|$ candidate molecules, but only one is active, and we have no prior knowledge about which one. Then there does not exist a learning algorithm that can outperform a random screen.
\begin{proposition} \label{thm:indiv-limit}
Assume $\f(x) = \Ib(x = x^\star)$ where $x^\star \sim \unif(1, \ldots, d)$, and there is no noise in the assay, so we measure $y_t = \f(x_t)$. Define expected regret, $R \triangleq \Eb_{x^\star}[\Eb_\pi[\sum_{t=1}^\infty \Ib(X_t \neq X^\star)]]$. A policy $\pi(x_t \mid \Dc_{t-1})$ that selects the next candidate $x_t$ uniformly at random from among remaining candidates has expected regret $(d-1)/2$. There does not exist a policy with lower expected regret.

\end{proposition}

\noindent Proof in \Cref{apx:proof-indiv-limit}. In short, any algorithm--using Bayesian optimization, reinforcement learning, generative models, agents, or anything else--will have $(d-1)/2$ failed experiments on average before finding the optimum. 
The problem is that on sparse problems there is no information to extrapolate from: an experiment only tells us whether or not the tested molecule works.

\paragraph{Our approach.} To break the $(d-1)/2$ regret barrier of \Cref{thm:indiv-limit}, we propose designing, synthesizing and testing mixtures of molecules, so that each experiment carries information about many molecules at once, as depicted in \Cref{fig:lids_overview} \citep{aldridge_group_2026, russo_learning_2018,liu_scalable_2025}.

\begin{figure}[ht!]
    \centering
    \includegraphics[width=1\linewidth]{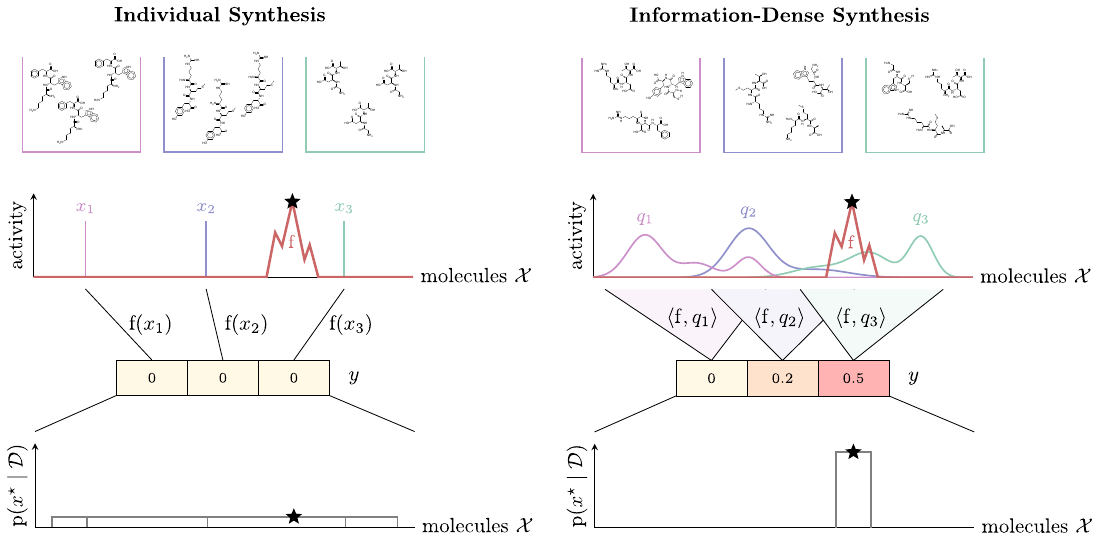}
    \caption{\textbf{Overview.} 
    Information-dense synthesis uses mixtures of molecules $q$, instead of single molecules $x$.
    By testing the mixtures as a pool, we encode the activity from many molecules into experiment outcomes $y$.
    Compared to individual synthesis, this can carry much more information about the optimal molecule $x^\star$, especially on sparse molecule-activity maps $\f$.
    } \label{fig:lids_overview}
\end{figure}

\section{Method}\label{sec:method}

We optimize $\f(x)$ by synthesizing and testing mixtures of molecules. 
Mathematically, we specify a mixture as a positive \textit{measure} $q(x)$: a distribution over different molecules, $\bar{q}(x) \in \Pc(\Xc)$, times a yield $0<c<\infty$, the total number of molecules that are made, $q(x) = c\cdot \bar{q}(x) \in \Mc^+(\Xc)$.
When we test the mixture, we assume the assay provides a noisy measurement of the average activity,
\begin{equation} \label{eqn:linearity-assumption}
	y_t = \int \f(x) q_t(x) \di x + \epsilon_t = \langle \f, q_t \rangle + \epsilon_t
\end{equation}
where $\epsilon_t$ is noise, and $\langle \f, \g \rangle \triangleq \int \f(x) \g(x) \di x$ denotes the inner product on functions.
In the case where we synthesize a pure compound, the distribution of molecules $q_t(x)$ reduces to a delta mass at a single point $x_t$, and we recover the standard problem setting, $y_t = \f(x_t) + \epsilon_t$.

\Cref{eqn:linearity-assumption} contains a linearity assumption: it says the experiment's outcome depends on the sum of the activity of each $x$, $\f(x)$, weighted by how much of $x$ is in the mixture, $q_t(x)$. This is an assumption about the \textit{assay}, not about the activity map $\f: \Xc \to \Rb$, which can be arbitrarily complex and nonlinear.
For example, in spectroscopy, the total signal depends on the sum of the absorbance or emission of each individual molecule, even if the absorbance and emission depend in a complex way on molecular structure. 
\textit{In vitro} or \textit{in vivo}, linearity violations are \textit{drug-drug interactions}, which are often lower order contributions to a cocktail's effect \citep{liu_scalable_2025}.

Our recipe for lab-in-the-loop learning consists of: \begin{enumerate}
    \item A synthesis model that describes the chemical process of creating a mixture. It specifies which $q(x) \in \Mc^+(\Xc)$ are practically achievable.
    \item An assay model that describes the noise associated with measurement, $\epsilon_t$.
    \item A design algorithm that chooses the next mixture $q_t$ to maximize information. \item An inference algorithm that updates our knowledge of $\f$ based on the experimental data.
\end{enumerate}

\subsection{Synthesis models}

We aim to design a complex mixture of molecules that is tractable to synthesize in the lab. To do so, we build on variational synthesis, which describes the output of stochastic synthesis using a generative model \citep{weinstein_optimal_2022,weinstein_manufacturing-aware_2026}. Our synthesis model consists of a family of measures $\Qc = \{q_\psi: \psi \in \Psi\} \subseteq \Mc^+(\Xc)$, where (1) $\psi$ denotes experimental parameters, i.e. specific instructions that can be sent to a synthesizer, and (2) $q_\psi(x)$ describes the mixture of molecules that come out of the synthesizer.
To design a mixture we choose a $\psi \in \Psi$.

We focus on models for solid phase synthesis, a combinatorial synthesis approach that is widely applied to make DNA, RNA, peptides, carbohydrates and more \citep{beaucage_deoxynucleoside_1981, caruthers_chemical_2013, merrifield_solid_1963, zuckermann_efficient_1992}. Consider synthesizing peptides of length $L$. 
The molecular space is $\Xc = \Bc^L$ where $\Bc$ is the alphabet of amino acids, size $B \triangleq |\Bc|$.
A solid phase peptide synthesizer makes peptides by sequentially adding amino acids. 
At each step it can add a single amino acid, or it can add a mixture of different amino acids, so that each growing peptide molecule will randomly encounter and react with a different amino acid. 
The resulting mixture can be described by a site-wise independent distribution $\cat(x \mid \beta_1) = \prod_{j=1}^L \beta_{1jx_{j}}$ where the probability of amino acid $x_j$ at position $j$ depends on its effective concentration in the $j$th reaction, $\beta_{1 j x_{j}}$, with $\beta_1 \in (\Delta_{B-1})^L$ \citep{houghten_mixture-based_1999}. 
By repeating the synthesis process $S$ times with different input concentrations, and mixing the products, we build up a distribution,
\begin{equation} \label{eqn:mix}
    \bar q_\psi(x) = \sum_{s=1}^S \alpha_s \cat(x \mid \beta_s),
\end{equation} where $\alpha \in \Delta_{S-1}$ specifies the relative concentration of peptides from each synthesis $s$, and $\beta \in (\Delta_{B-1})^{S \times L}$. This is a mixture model over discrete sequences. 
We set $\Qc= \{q_{\psi=(\alpha, \beta, c)}: c \sum_{s=1}^S \alpha_s \cat(x \mid \beta_s)\}$ as the family of synthesizable distributions, where $\psi=(\alpha,\beta, c)$ are the experimental parameters, including the overall yield $c$. 
This synthesis model is applicable to a range of heteropolymers, e.g. for DNA we can set $\Bc$ to be the nucleotides \citep{weinstein_optimal_2022,weinstein_manufacturing-aware_2026}. 

The laboratory time and resources required to synthesize the mixture distribution in \Cref{eqn:mix} is effectively the same as that required to synthesize $S$ pure sequences $x_1, \ldots, x_S$. The number of reaction steps and the amount of reagents are the same. 
The only added complexity is from mixing reagents before or during each synthesis step, a process amenable to automated liquid handling \citep{torres-acosta_automated_2022}.

\subsection{Assay models}
We assume a model of the assay's noise. We focus on Gaussian noise, $\epsilon_t \sim \mathrm{Normal}(0, \sigma)$. Though not required, our method exploits prior knowledge of the noise level, e.g. from control experiments, so we assume the standard deviation $\sigma$ is known.
Bernoulli or Poisson noise is also possible.

\subsection{Designing mixtures}

We work in the framework of sequential Bayesian experimental design \citep{rainforth_modern_2024}. We optimize synthesis to maximize the information our next experiment  will deliver about the optimum.
This acquisition function has been proposed in predictive entropy search and information directed sampling \citep{hernandez-lobato_predictive_2014, russo_learning_2018}.
First, we place a prior on the unknown molecule-activity map,
\begin{align}
    \theta &\sim \pr(\theta) \label{eqn:prior}\\
    Y_t &\sim \pr(y \mid \langle \f_\theta, q_{t} \rangle ) \label{eqn:like}
\end{align}
After performing $t-1$ experiments, our posterior updates to $\pr(\theta \mid \Dc_{t-1} = \{(y_1, q_1), \ldots, (y_{t-1}, q_{t-1})\})$.
To choose the next mixture, we maximize the expected information gain (EIG) about the optimum: $q_t = \underset{q_\psi \in \Qc}\amax \,\EIG_{x^\star}(q_\psi)=$
\begin{align} \label{eqn:info-gain}
	\underset{q_\psi \in \Qc}\amax[\Hc(\pr(x^\star_\theta\mid \Dc_{t-1})) - \Eb_{\pr(y_t \mid  q_\psi, \Dc_{t-1})}[\Hc(\pr(x^\star_\theta \mid \Dc_{t-1} \cup \{(Y_t, q_\psi)\}))]]
\end{align}
where $\Hc$ denotes the Shannon entropy and $x^\star_\theta \triangleq \amax_{x} \f_\theta(x)$. This objective quantifies the reduction in uncertainty about $x^\star$ after observing $y_t$, on average across predicted $y_t\sim \int \pr(y_t \mid q_{\psi}, \theta) \pr(\theta \mid \Dc_{t-1}) d\theta$. 
It is equivalent to the mutual information between the new datapoint, $(y_t, q_t)$, and the optimal molecule $x^\star$ \citep{hennig_entropy_2012,hernandez-lobato_predictive_2014,hernandez-lobato_predictive_2015,jones_efficient_1998,russo_learning_2018}. Assays are often run in parallel, e.g. in a 96 well plate format, in which case we can consider a batch version of the objective, with $\Dc_{t-1} \cup \{(Y^t_1, q_{\psi_1}), \ldots, (Y^t_E, q_{\psi_E})\}$.

The EIG is critical to exploiting stochastic synthesis. Thompson sampling draws $\theta \sim \pr(\theta \mid \Dc_{t-1})$ and chooses the next design as $q_{t} = \amax_{q \in \Qc} \langle \f_\theta, q \rangle\ = \delta_{x^\star_\theta}$, 
i.e. it reverts to synthesizing and testing individual molecules \citep{thompson_likelihood_1933}.
The upper confidence bound algorithm and expected improvement produce the same degeneracy (\Cref{app:acquisition_functions}) \citep{auer_finite-time_2002, jones_efficient_1998}.
As a result, these standard acquisition functions fail to provide major speedups on sparse problems, even when given access to stochastic synthesis (\Cref{sec:results}) \citep{russo_learning_2018}. \paragraph{Approximating the information gain}
The EIG is difficult to compute explicitly. Instead, we optimize the Barber-Agakov variational bound \citep{foster_unified_2020,barber_information_2003}, \begin{equation}\label{eq:batch-EIG}
	\EIG_{x^\star}(q_{\psi^t_{1:E}}) \geq \Hc(\pr(x^\star_\theta \mid \Dc_{t-1})) + \Eb_{\pr(\theta \mid \Dc_{t-1})\pr(\epsilon)}[\log r_\phi(x^\star_\theta \mid \langle \f_\theta, q_{\psi^t_{1:E}} \rangle + \sigma\epsilon_{1:E})]
\end{equation} $\triangleq \Ic_{BA}(x^\star ; q_\psi, \phi)$, where $r_\phi(x^\star \mid y)$ is the approximate posterior and $\pr(\epsilon)=\mathrm{Normal}(0,1)$ is the reparameterized noise (\Cref{app:variational_lower_bound}). We first draw samples from the posterior $\theta_{1:M} \overset{iid}{\sim} \pr(\theta \mid \Dc_{t-1})$ (\Cref{sec:inference}), then jointly optimize the synthesis $q_\psi$ and the approximate posterior $r_\phi$ by stochastic gradient ascent, using AdamW \citep{loshchilov_decoupled_2018, foster_unified_2020}. Details on architecture and training are in \Cref{app:variational_dist}.
This variational bound can be understood as an encoder-decoder model \citep{grover_uncertainty_2019}. The laboratory experiment \textit{encodes} the unknown molecule-activity map $\f_\theta$ into a noisy E-dimensional measurement $y^t_{1:E} = \langle \f_\theta, q_{\psi^t_{1:E}} \rangle + \sigma \epsilon_{1:E}$. The approximate posterior $r_\phi$ \textit{decodes} the optimal molecule $x^\star$ from this measurement.

\paragraph{A tractable sparse function class}

We need a function class and prior that can describe sparse molecule-activity maps, which can be challenging to model with Gaussian processes \citep{park_jump_2022, meng_combined_2021}.
We consider a single-layer Bayesian neural network with an exponential nonlinearity,
$\f_{\theta = (W,V)}(x) = W^\top \exp(V \cdot x)$
where $x \in \Bc^L$ is represented as a one-hot encoding, $W \in \Rb^N$, $V \in \Rb^{N\times L \times B}$, and the dot product is over the length dimension $L$ and alphabet dimension $B$.
We design a prior that encourages sparsity, so only $x$ that closely match one of the filters $V_1, \ldots, V_N$ produce nontrivial outputs, 
and assume that we have an informative prior on the maximum activity that the assay can record, $\|\f\|_\infty$, based on e.g. positive control experiments. 
This function class is expressive, and offers closed-form computation of the integral $\int \f(x) q_\psi(x) \di x$, as well as fast approximation of $x^\star$ (\Cref{app:bnn}).

\subsection{Inference from mixtures} \label{sec:inference}
We perform automated approximate Bayesian inference in a probabilistic programming language (NumPyro), adaptively switching from MCMC to stochastic variational inference as the posterior concentrates (\Cref{app:posterior_approx}).
When done gathering data, we approximately generate the maximum \textit{a posteriori} molecule, $\hat{x}^\star = \amax_x \pr(x^\star_\theta = x \mid \mathcal{D}_t)$, as our best guess of the optimal molecule $x^\star$. The full procedure is summarized in \Cref{alg:LIDS}.

\begin{algorithm}[H]
\caption{\textbf{L}earning with \textbf{I}nformation-\textbf{D}ense \textbf{S}ynthesis (LIDS)}\label{alg:LIDS}
\begin{algorithmic}
\Require batch size $E \in \Nb$, syntheses $S \in \Nb$, noise $\sigma$, posterior samples $M \in \Nb$, prior $\pi(\theta)$.

\While{conducting experiments, $t=1,2,3,\ldots$}
\State $q^t_{1:E}, \phi_t = \amax \,\Ic_{BA}(x^\star ; q_{\psi}, \phi)$ \Comment{Maximize EIG}
\State $y^t_{1:E} = \langle \f, q^t_{1:E} \rangle + \sigma\epsilon^t_{1:E}$ \Comment{Run batch of experiments}
    \State $\Dc_t \gets \Dc_{t-1} \cup \{(y^t_{i}, q^t_{i})_{i=1}^{E} \}$
    \State $\theta^t_{1:M} \overset{iid}{\sim} \pr(\theta \mid \Dc_t)$ \Comment{Update beliefs}
\EndWhile
\State \Return $\hat{x}_t^\star = \amax_x \frac{1}{M}\sum_{j=1}^M\Ib[x = x^\star_{\theta^t_j}]$

\end{algorithmic}
\end{algorithm}

\section{Related Work}

Our approach extends group testing, which searches among many candidates by conducting pooled tests, to a high-dimensional Bayesian optimization setting \citep{aldridge_group_2026,Garnett2023-uf}.
Group testing methods have a long history in molecular discovery and combinatorial chemistry \citep{Kainkaryam2009-bb}.
They are closely related to compressed sensing, which exploits structure such as sparsity to learn from limited measurements \citep{Atia2009-gq}.
\citet{liu_scalable_2025} propose \textit{compressed screening} to learn the activity of a set of compounds synthesized individually, using pooled tests \citep{Yao2023-hd,Cleary2020-zm}.
\citet{Weinstein2025-fv} generalize from learning vectors to training sequence activity models; they synthesize DNA stochastically by variational synthesis, then deliver multiple DNA sequences into the same cell to make randomized mixtures.
We extend these ideas, learning sequence-activity models through iterative experimentation with model-designed mixtures.

A critical advantage of pooled testing is that it does not require barcoding individual molecules \citep{Peterson2023-ob}.
LIDS is therefore not restricted to DNA and DNA-encoded libraries, and can instead be applied to other classes of molecules \citep{Szostak1997-wh}.

LIDS builds on previous methods for optimization based on the EIG, such as entropy search \citep{mackay_information-based_1992,villemonteix_informational_2009,hennig_entropy_2012,hernandez-lobato_predictive_2014}.
\citet{russo_learning_2018} show the EIG enables rapid optimization of sparse functions, where UCB and TS fail.
To operationalize this idea, we apply recent advances in EIG approximation \citep{rainforth_modern_2024, foster_unified_2020}.
Like the uncertainty autoencoder (UAE) \citep{grover_uncertainty_2019}, we design experiments end-to-end by training an encoder and decoder, but our focus is on learning sparse functions $\f(x)$ rather than sparse Euclidean vectors.
Like \citet{sharrock_wasserstein_2026} we optimize distributions $q(x)$, but our goal is to design mixtures not ease computation.

Molecular discovery relies on causal assumptions: a molecule-activity map $\f(x)$ specifies the effect of an intervention $x$ on an outcome $y$. 
We disaggregate the conventional causal model, $y = \f(x)+\epsilon$, and replace it with a hierarchical causal model, where many molecules (subunits) affect an outcome (\Cref{fig:hcm}) \citep{Weinstein2026-jl,Weinstein2024-oq}.
The corresponding estimator, \Cref{eqn:prior,eqn:like}, is Bayesian distributional regression \citep{zaheer_deep_2018,Szabo2016-fd}.

\section{Theory}\label{sec:theory}
We show information-dense synthesis can rapidly search large molecular spaces (\Cref{sec:spec_cases}), but is slowed by experimental limitations (\Cref{sec:lim_n_trad})

\subsection{Exponential speed ups}\label{sec:spec_cases}
In principle, information-dense synthesis can find the optimal molecule with a single experiment.

\begin{example}[No noise, discrete map]\label{ex:O1-search}
    Assume activity is discrete, s.t. $\f: \Xc \to \Yc$ for $\Yc=\{0, \ldots |\mathcal{Y}|-1\}$. 
    Assume the model is well specified, $\f \in \{\f_\theta : \theta \in \Theta\}$, and $\pi(\f_\theta = \f) > 0$, 
    we can synthesize any $q \in \Qc = \Pc(\Xc)$, and our assay has no noise, $y_t = \langle \f, q\rangle$. 
    Then, by optimizing the EIG (\Cref{eqn:info-gain}), we find $x^\star$ after one experiment.
\end{example}
\begin{proofsketch}
Set $q(x) = \frac{1}{Z}\frac{1}{|\Yc|^x}$ where $Z = \sum_{x=1}^d \frac{1}{|\Yc|^x}$.
Now $Z y = Z \langle \f, q \rangle = \sum_{x=1}^d \frac{\f(x)}{|\Yc|^x}$, so $\f(x)$ is encoded in the digits of $Z y$, e.g. for $|\Yc| = 10$ we have $Z y = 0.\f(1)\f(2)\ldots\f(d)$. 
\end{proofsketch}
\noindent Proof in \Cref{app:O1-proof}. This example assumes discrete activity, e.g. $y$ could be the number of fluorophores that attach to $x$. However, $\mathcal{O}(1)$ optimization can also be achieved for $\Yc = \Rb$, when $\f_\theta$ is sparse (\Cref{app:1-sparse_example}).
Next, consider assays that return just one bit, instead of a continuous value. \begin{example}[Binary activity test]\label{ex:2E-ary_search}
    Assume $\f$ is 1-sparse, $\f(x) \in \{w \Ib(x = \theta): \theta\in\Xc,\ w\in\Rb_+\}$, we can synthesize any $q \in \Qc = \Pc(\Xc)$, and our model is well-specified, with uniform prior $\pi(\theta) = 1/d$. The assay reports where any molecule is active, $y_t = \Ib[\langle \f, q\rangle > 0]$.
Then, by optimizing the EIG with batch size $E$ (\Cref{eq:batch-EIG}), we find $x^\star$ after $\lceil \log_2 d/E \rceil$ batched experiments. \end{example}
\begin{proofsketch}
The EIG executes $2^E$-ary search. Each $q$ contains half the current possible $x^\star$ values. Observing $y$ tells us whether $x^\star$ is in that half or the other, cutting the search space in half.
\end{proofsketch}
\noindent Proof in \Cref{app:2E-ary_search-proof}, with an additional example in \Cref{app:proof-info-gain-thresh}. The space of all small molecules is estimated at $d = {\sim} 10^{60}$ \citep{ruddigkeit_enumeration_2012}. This suggests it could be searched with three 96-well plates, where testing compounds individually would require ${\sim}10^{60}$ experiments (\Cref{thm:indiv-limit}). We next discuss how reality intervenes.

\subsection{Limitations and tradeoffs}\label{sec:lim_n_trad}
Our analytic examples assumed no noise, full control over the molecular mixture, and restrictions on the molecule-activity map. In this section we examine the method's behavior as these assumptions are relaxed or changed. We first examine the information gained in a single experiment.

\begin{proposition}[Information bound]\label{prop:gaussian_mi_bound}
Assume Gaussian noise, $y = \langle \f_\theta, q \rangle + \sigma \epsilon$, where $\epsilon \sim \mathrm{Normal}(0,1)$, and $\langle \f_\theta, \bar q \rangle \in \Zc$ for $\Zc$ a finite set. 
Then,
\begin{align}
        \EIG_\theta(q) \leq \min\left\{ \Hc(\langle \f_\theta, \bar{q}\rangle), \, \frac{1}{2}\log \left(1 + \frac{c^2}{\sigma^2} \mathrm{Var}_{\pi} [\langle \f_\theta, \bar{q} \rangle] \right) \right\}. \end{align}
\end{proposition}
\noindent The proof (\Cref{app:proof-EIGbound}) applies the Shannon-Hartley theorem.
This bound shows information is limited by an assay's \textit{signal-to-noise ratio} $\textsc{SNR}(q) \triangleq \frac{c^2}{\sigma^2} \mathrm{Var}_{\pi} [\langle \f_\theta, \bar{q} \rangle]$.
High noise $\sigma$ can reduce the EIG to zero. 
High synthesis yield $c$ can increase the EIG: if we have more material to test, the SNR will increase.
Control over synthesis is also important: we must achieve a high inner product between what we can make, $q$, and the molecule-activity map we want to learn, $\f_\theta$.

This bound also reveals a tradeoff that comes from using stochastic synthesis rather than individual synthesis. We can often increase the signal's entropy $\Hc(\langle \f_\theta, \bar{q}\rangle)$ by testing more spread out mixtures $\bar q(x)$. However, the SNR is maximized by concentrating on a single molecule, since $\exists\delta_{x} \in\amax_{\bar q}\mathrm{Var}_{\pi} [\langle \f_\theta, \bar{q} \rangle]$, following Bauer's maximum principle (\Cref{app:proof-EIGbound2}).\hspace*{\fill}
\begin{wrapfigure}{r}{0.5\textwidth}
\centering
    \includegraphics[width=\linewidth]{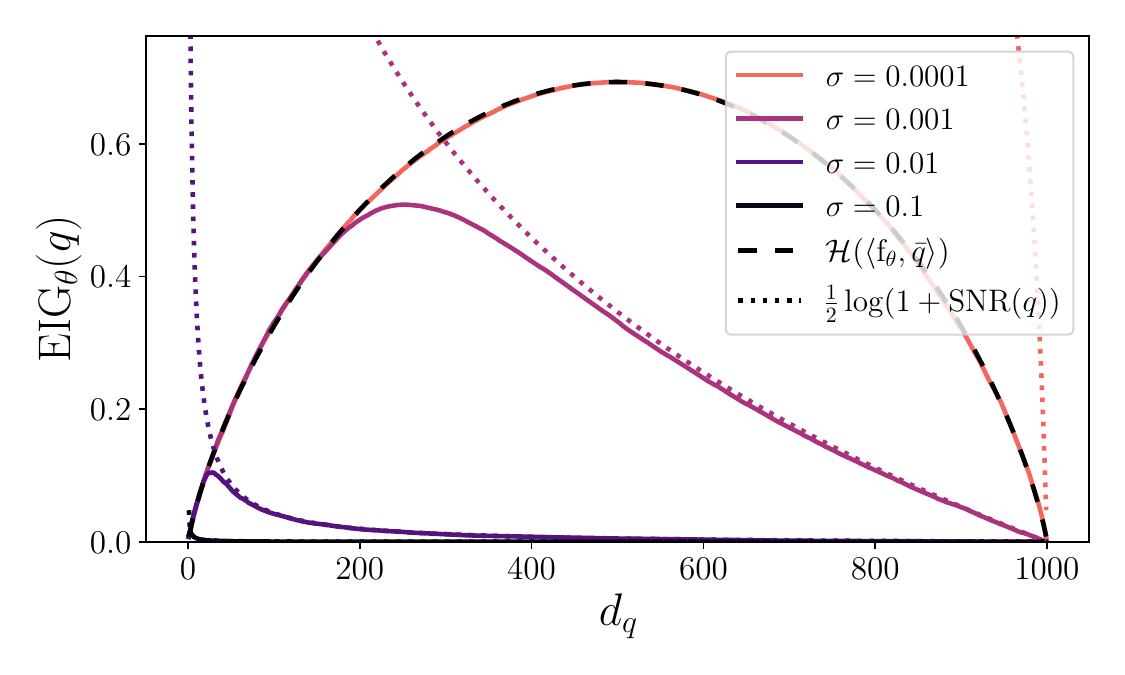}
    \caption{\textbf{EIG tradeoffs.} EIG from \Cref{ex:gaussian_example}, plotted alongside MI upper bounds from \Cref{prop:gaussian_mi_bound}, as a function of $d_q = |\mathrm{supp}(q)|$.}
    \label{fig:MI_vs_dq}
    \vspace{-2em}
\end{wrapfigure}
\begin{example}\label{ex:gaussian_example}
    Assume the molecule activity map is 1-sparse with maximum value one, $\f_\theta(x) \in \{ \Ib(x = \theta):\theta\in\Xc \}$, and the prior is uniform, $\pi(\theta) = 1/d$.
    Assume a synthesis model that creates uniform mixtures, $q(x) = \frac{1}{d_q} \Ib(x \in U)$ for $U \subseteq \Xc$, $\|q\|_1=c=1$ and $d_q \triangleq |U|$. 
    Assume Gaussian noise $y = \langle \f_\theta, q\rangle + \sigma \epsilon$ where $\epsilon \sim \mathrm{Normal}(0, 1)$.
    In this example, we can numerically solve the EIG, and compare with the bounds from \Cref{prop:gaussian_mi_bound}, see \Cref{fig:MI_vs_dq}.
    As the noise $\sigma$ increases, the optimal tradeoff between entropy and SNR shifts to smaller values of $d_q$.
    See \Cref{app:proof-gaussian_example} for details.
\end{example}
Synthesizing mixtures can offer major gains in early experiments, but the advantage over individual synthesis decreases to zero in the limit as the number of experiments goes to infinity.
\begin{proposition}\label{prop:paninski}
    Let $\f_\theta : \Xc \to \Rb_+$ be a model of the molecule-activity map, where $\{\f_\theta =(\f_\theta(x))_{x\in\Xc}: \theta \in \Theta\} = \Rb^{|\Xc|}$. Assume the prior $\pi(\f_\theta)$ is positive and absolutely continuous in a neighborhood of the true $\f$. Consider maximizing the EIG for $\f_\theta$. Let $\Mc^+_c(\Xc) = \{q = c \bar q: \bar q \in \Pc(\Xc)\}$ be the set of mixtures and let $\Sc_c(\Xc) = \{q = c\delta_x: x \in \Xc\}$ be the set of delta mixtures (individual synthesis).
    Define $\sigma_{info}^2(\Qc) \triangleq \underset{t\to\infty}{\lim}t\cdot\mathrm{Cov}[\pr(\f_\theta \mid \Dc_t)]$ to be the asymptotic covariance of the posterior when we choose $q \in \Qc$ at each step $t$ by maximizing the EIG.
    Then, the asymptotic covariance is equal regardless of whether we test mixtures or individual molecules: 
$\sigma^2_{info}(\Sc_c(\Xc)) = \sigma^2_{info}(\Mc_c^+(\Xc))$.
\end{proposition}
\noindent The proof is in \Cref{app:proof-paninski}; it uses the Bernstein-von Mises theorem for sequential Bayesian design from \citet{paninski_asymptotic_2005}, which implies the posterior over $\f_\theta$ is asymptotically normal.
A limitation of this result is it applies to the EIG for learning the entire function $\f_\theta$, while LIDS focuses just on the optimum $x^\star$; we conjecture similar behavior for learning the optimum, since it is a deterministic function of $\f_\theta$.
The overall picture is that at the beginning of learning, synthesizing mixtures can offer major advantages, but in the limit as we collect infinite data and approach complete knowledge of $\f$, the advantages diminish.

\section{Empirical Results}\label{sec:results}

We compare LIDS to individual synthesis empirically. We evaluate in simulations, where the ground truth molecule-activity map $\f$ is specified by different oracles, and measurements have Gaussian noise with known variance.
We compare to molecular optimization methods that use individual synthesis.
First, Thompson sampling using the same prior $\pi(\theta)$ and model $\f_\theta$ as LIDS (\textit{Thompson}) \citep{thompson_likelihood_1933}. 
Then, Gaussian process (GP) Bayesian optimization with expected improvement (EI), using either the exponential Hamming kernel (\textit{Hamming EI}) or the LOCK kernel (\textit{LOCK EI}), which is SOTA for learning protein activity from limited data \citep{jankowiak_flexible_2026,amin_biological_2023,groth_kermut_2024}.
Finally, we ablate LIDS by removing stochasticity from its synthesis design $q_t$, replacing it with $\bar q_t(x) \gets \amax_{\delta_x} \langle \delta_x, \bar q_t \rangle$ (\textit{LIDS-argmax}). 
To compare LIDS fairly to individual synthesis, we set the number of syntheses equal to the batch size $S=E$ and only vary the weights within a batch $\psi^t_{1:E} = (\alpha^t_1, \beta, c),...,(\alpha^t_{E}, \beta, c)$, so the synthesis cost (amount of reagents, number of reaction steps) is the same (Note when experiments are bottlenecked by assay cost rather than synthesis cost, LIDS can use $S \gg E$ to reach higher EIG).
To monitor performance, after every round $t$ we evaluate the true activity of the predicted optimum, $\f(\hat{x}_t^\star)$. 
All experiments use a single Nvidia A100 GPU. LIDS is by far the most computationally intensive: EIG maximization takes $\approx30$min and posterior updates take $\approx30$min for the oracles in \Cref{sec:res_plm}. Key results:
\begin{itemize}
    \item On synthetic $\f$ with rare activity, LIDS finds the optimum in 2-5 experiments (median) where individual synthesis methods fail even after 100.
    \item On learned $\f$ from a protein language model, LIDS optimizes an order of magnitude faster.
    \item LIDS is robust to model misspecification and to noise in the assay or synthesis, with performance degrading gradually towards individual synthesis as noise increases.
\end{itemize}

\subsection{Sparse synthetic oracle}\label{sec:res_sparse}

We first evaluate LIDS on a small-scale problem where we specify a sparse $\f$ as an oracle.

\paragraph{Setup} Let $\Xc = \{A,T,C,G\}^8$ be the set of all $d = |\Xc| = 65{,}536$ DNA sequences of length $L=8$. We construct an oracle $\f$ where the optimal sequence $x^\star$ has activity one and all sequences more than one mutation away from $x^\star$ have activity zero. Measurement adds Gaussian noise with standard deviation $\sigma \in \{10^{-2}, 10^{-3},10^{-4}\}$. We fix the yield $c=1$. 
This corresponds to \textit{peak-signal-to-noise ratio} $\textsc{PSNR} \triangleq 10\log_{10}(\frac{c^2}{\sigma^2}) \in \{40,60,80\}$ dB.
At every round, we test a batch of $E=12$ designs.
We run 10 independent repeats for each method.
Details in \Cref{app:fsp}.

\begin{figure}[ht!]
    \centering
    \includegraphics[width=1\linewidth]{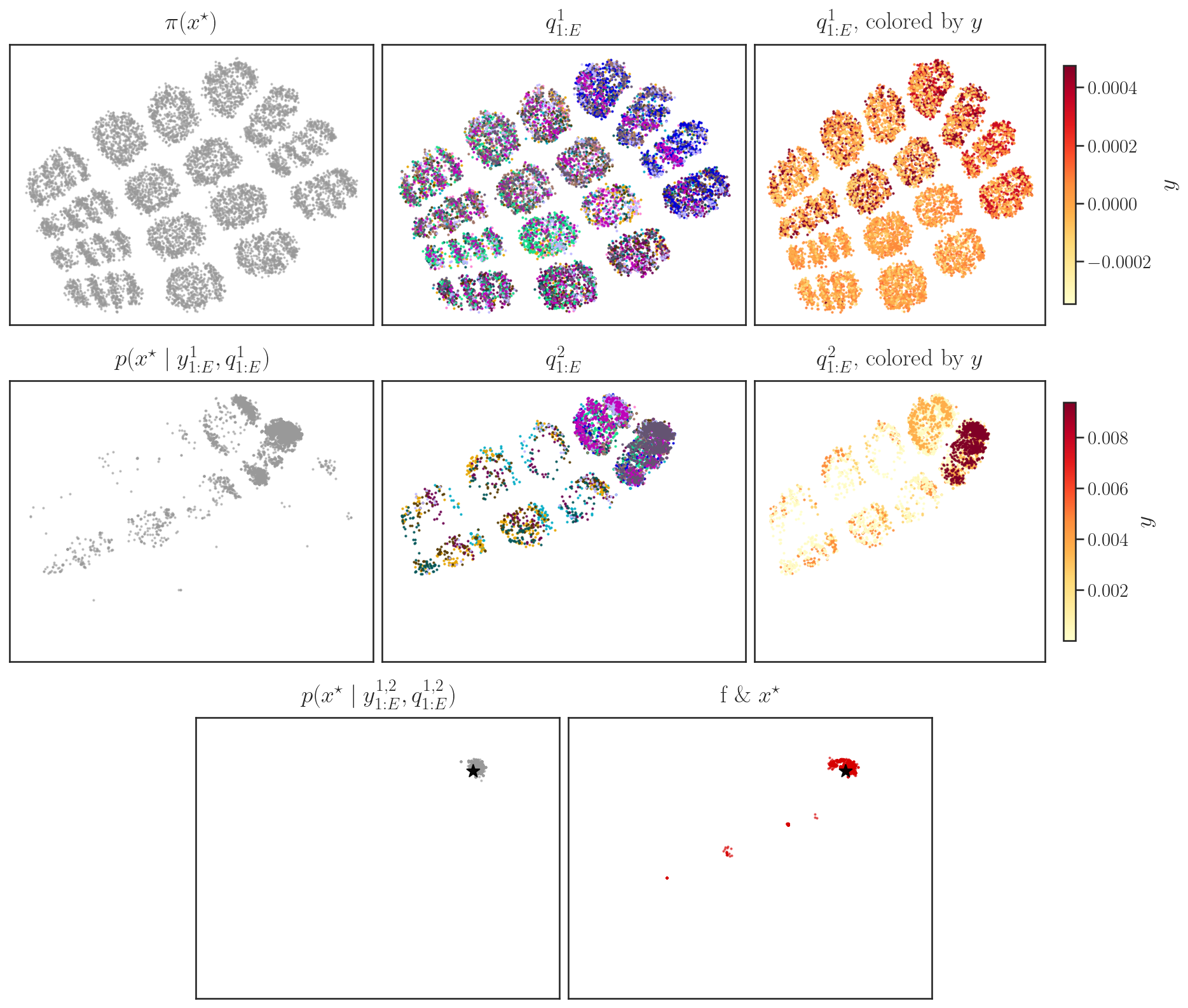}
    \caption{\textbf{Two rounds of LIDS on a sparse synthetic oracle (\Cref{sec:res_sparse}).} 
    UMAP projection of DNA sequences. 
    Row 1: samples from the prior $\pi(x^\star)$ (left), the mixtures designed at $t=1$ by LIDS $q^1_{1:E}$ (middle), and the same mixtures, colored by their measured activity $y^1_{1:E}$ (right). 
    Row 2: the posterior after one experiment round, $\pr(x^\star \mid \Dc_{1}=(q^1_{1:E},y^1_{1:E}))$ (left), the mixtures designed at $t=2$ by LIDS $q^2_{1:E}$ (middle), and the same mixtures, colored by their activity $y^2_{1:E}$ (right).
    Row 3: the posterior after two experiment rounds, $\pr(x^\star \mid \Dc_{2}=(q^{1,2}_{1:E},y^{1,2}_{1:E}))$ (left), and the true activity (right), normalized to be a distribution for visualization $\f(x) / \sum_{x\in\Xc}\f(x)$.
    After these two rounds, LIDS has converged on the global optimum (star).
    } \label{fig:lids_viz}
\end{figure}

\paragraph{Results} At low noise-levels ($\textsc{PSNR} = 80$ dB), LIDS identifies the optimum $x^\star$ in less than 8 rounds of experiments across all 10 repeats, with a median of just 2 rounds (\Cref{tab:rounds_sparse_dna,fig:DNA_low}) In most repeats, the individual synthesis methods (Thompson sampling and Hamming EI) failed to find the optimum within 100 rounds of experiments.
As the assay noise increases, the advantage of LIDS diminishes gradually, approaching the individual synthesis methods (\Cref{fig:DNA_medium,fig:DNA_high}). 

We visualized LIDS qualitatively by a low-dimensional projection of its designs $q_t$ and beliefs $\pr(x^\star \mid \mathcal{D}_t)$ (\Cref{fig:lids_viz}).
At each round $t$, LIDS designs complex stochastic syntheses that cover its remaining uncertainty.

\begin{table}
\centering
\begin{minipage}[t]{0.48\linewidth}
\centering
\small \captionof{table}{\textbf{Sparse oracle.} Number of rounds until finding the optimal molecule, $\hat{x}^\star = x^\star$.}
\label{tab:rounds_sparse_dna}
\begin{tabular}{@{}clccc@{}}
\toprule
\multirow{2}{*}{Noise level} & \multicolumn{1}{c}{\multirow{2}{*}{Method}} & \multicolumn{3}{c}{\# of rounds $\downarrow$} \\ \cmidrule(l){3-5} 
                             & \multicolumn{1}{c}{}                        & min      & median     & max     \\ \midrule
\multirow{3}{*}{Low}         
& Thompson                                 & 14                 & $>$100            & $>$100            \\
& Hamming EI                               & 42                 & $>$100            & $>$100            \\
& LIDS (ours)                              & \textbf{2}         & \textbf{2}        & \textbf{8}        \\
\midrule
\multirow{3}{*}{Medium}      
& Thompson                                 & 19                 & $>$100            & $>$100            \\
& Hamming EI                               & 45                 & $>$100            & $>$100            \\
& LIDS (ours)                              & \textbf{3}         & \textbf{4.5}      & \textbf{14}       \\
\midrule
\multirow{3}{*}{High}        
& Thompson                                 & 38                 & $>$100            & $>$100            \\
& Hamming EI                               & 41                 & $>$100            & $>$100            \\
& LIDS (ours)                              & \textbf{21}        & $>$100            & $>$100            \\
\bottomrule
\end{tabular}
\end{minipage}\hfill
\begin{minipage}[t]{0.48\linewidth}
\centering
\small
\captionof{table}{\textbf{PLM oracle.} Number of rounds until finding the optimal molecule, $\hat{x}^\star = x^\star$.}
\label{tab:rounds_sparse_aa}
\begin{tabular}{@{}clccc@{}}
\toprule
\multirow{2}{*}{Protein} & \multicolumn{1}{c}{\multirow{2}{*}{Method}} & \multicolumn{3}{c}{\# of rounds $\downarrow$} \\ \cmidrule(l){3-5} 
                             & \multicolumn{1}{c}{}                        & min       & median    & max      \\ \midrule
\multirow{5}{*}{PETase}      
& Thompson                                  & $>$100            & $>$100            & $>$100            \\
& Hamming EI                                & 9                 & 27                & 88                \\
& LOCK EI                                   & 12                & 27                & 47       \\
& LIDS-argmax                               & $>$5              & $>$5              & $>$5              \\ 
& LIDS (ours)                               & \textbf{2}        & \textbf{3}        & $>$5              \\
\midrule
\multirow{5}{*}{RuBisCO}     
& Thompson                                  & 72                & $>$100            & $>$100            \\ 
& Hamming EI                                & 42                & 118               & $>$200            \\ 
& LOCK EI                                   & 81                & $>$100            & $>$100            \\ 
& LIDS-argmax                               & $>$5              & $>$5              & $>$5              \\ 
& LIDS (ours)                               & \textbf{2}        & \textbf{2}        & $>$5              \\
\midrule
\multirow{5}{*}{SazCA}     
& Thompson                                  & $>$100            & $>$100            & $>$100            \\ 
& Hamming EI                                & $>$100            & $>$100            & $>$100            \\ 
& LOCK EI                                   & $>$100            & $>$100            & $>$100            \\ 
& LIDS-argmax                               & $>$5              & $>$5              & $>$5              \\ 
& LIDS (ours)                               & $>$5              & $>$5              & $>$5              \\
\bottomrule
\end{tabular}
\end{minipage}
\end{table}

\begin{figure}
  \begin{subfigure}{0.33\textwidth}
    \includegraphics[width=\linewidth]{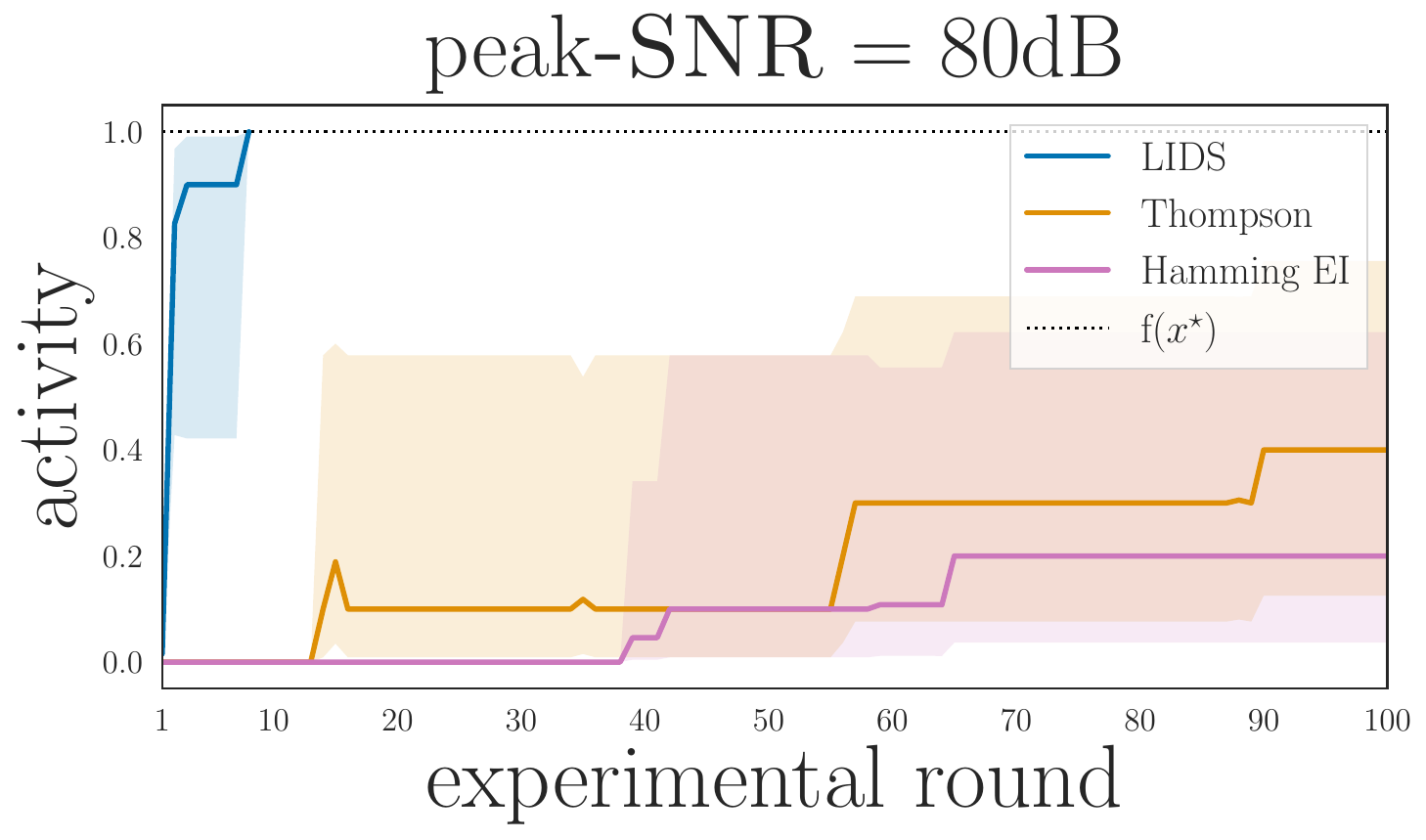}
    \caption{}
    \label{fig:DNA_low}
  \end{subfigure}\hfill
  \begin{subfigure}{0.33\textwidth}
    \includegraphics[width=\linewidth]{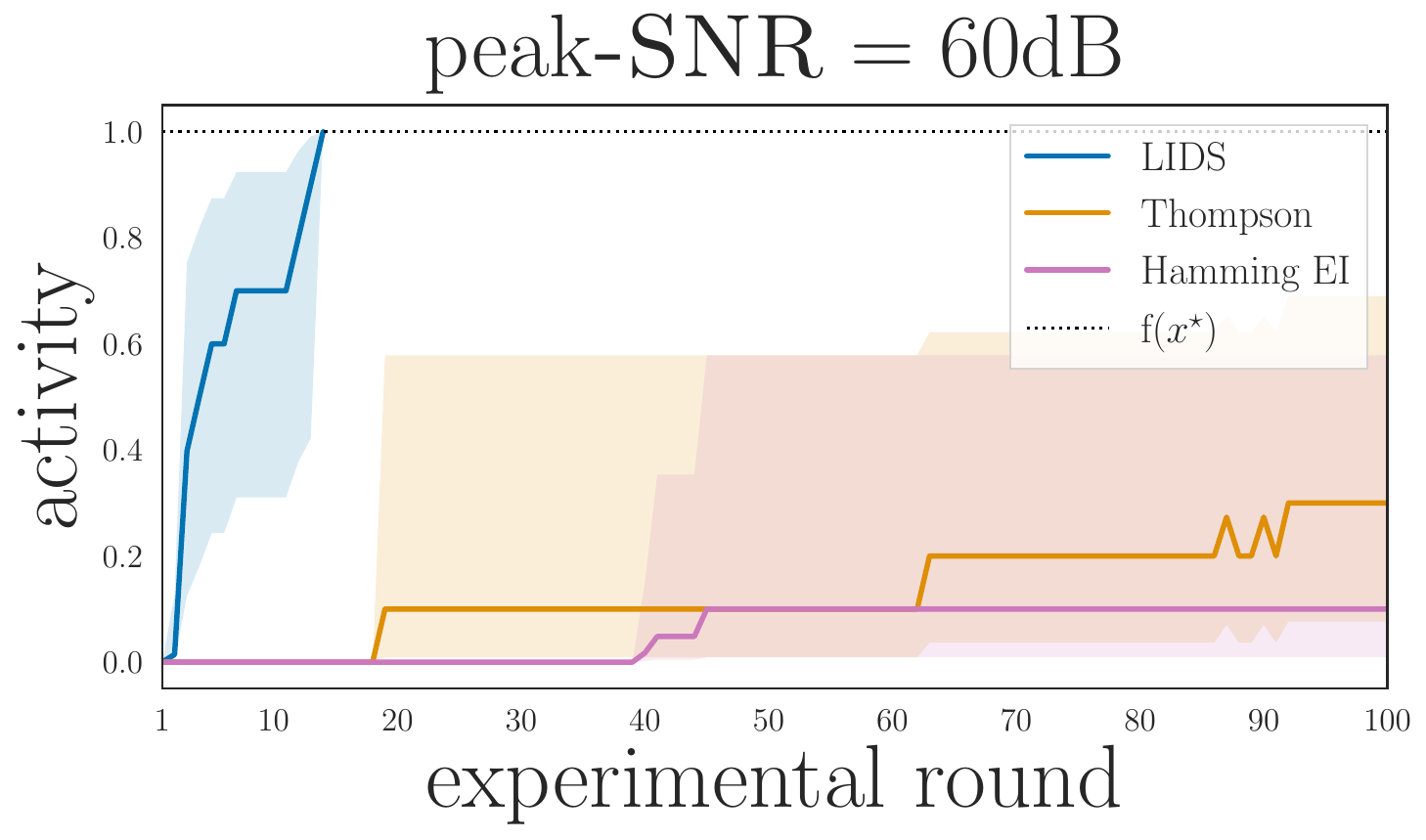}
    \caption{}
    \label{fig:DNA_medium}
  \end{subfigure}\hfill
  \begin{subfigure}{0.33\textwidth}
    \includegraphics[width=\linewidth]{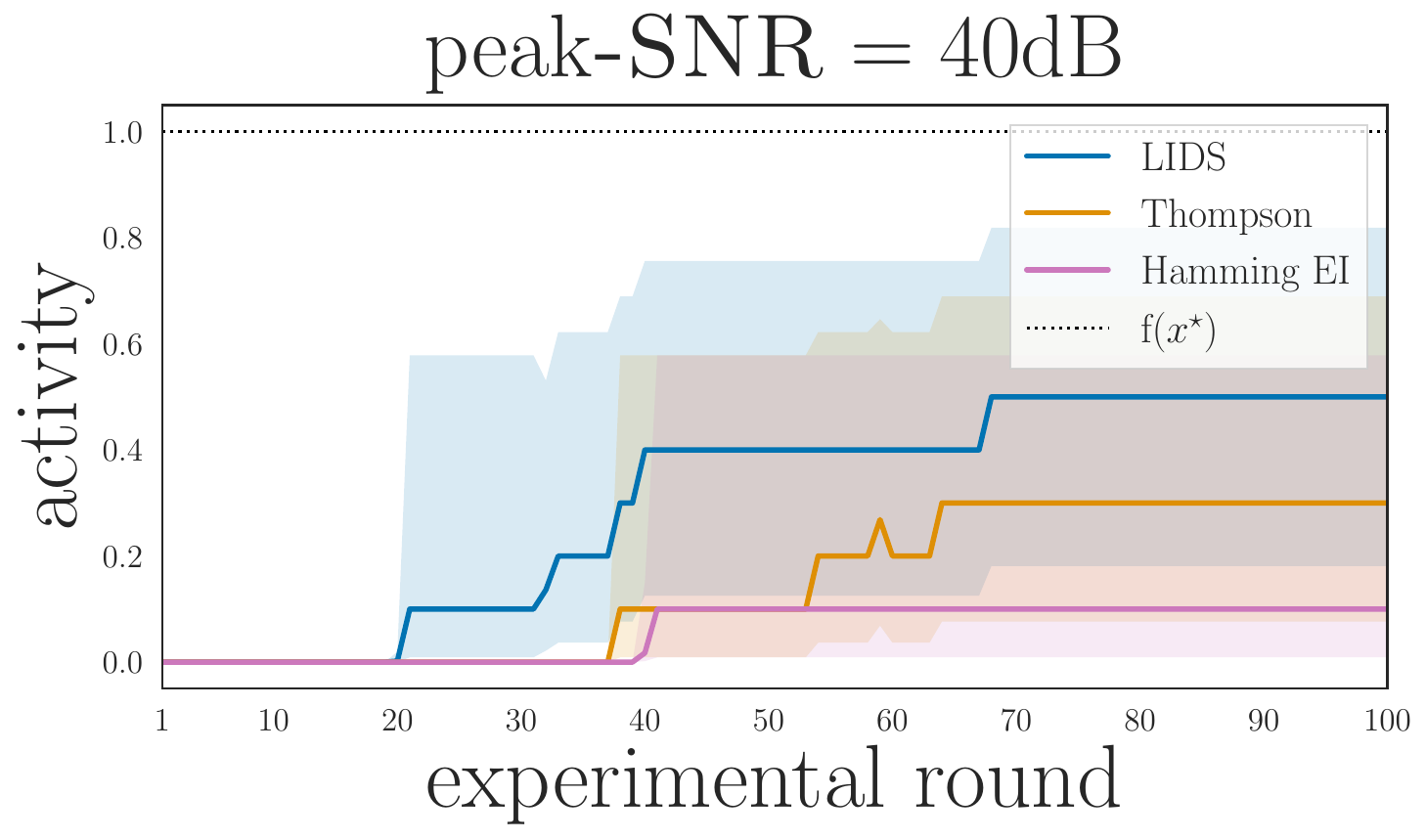}
    \caption{}
    \label{fig:DNA_high}
  \end{subfigure}
\vspace{-1em}
  \caption{\textbf{Synthetic $\f$ with rare activity.} LIDS, Thompson, and Hamming EI, at varying noise-level $\sigma$, optimizing a synthetic oracle. The activity of the predicted optimum, $\f(\hat{x}^\star_t)$, is plotted as a function of experimental round $t$, averaged over 10 duplicates. The 95\% CI is calculated using a logit transform and the delta method, so that it stays bounded in the interval $[0,\f(x^\star)]$ (See \Cref{app:CI}). }
  \label{fig:DNA_empirical}
\vspace{-1em}
\end{figure}

\subsection{Protein language model oracle}\label{sec:res_plm}

We next evaluate LIDS on a peptide discovery problem, where the oracle molecule-activity map $\f$ is estimated from large scale data.

\paragraph{Setup}

Let $\Xc = \Bc^L$ be the set of peptides of length $L$, where $\Bc=\{A,C,...\}$ is the 20 natural amino acids. 
As an oracle $\f: \Xc \to \Rb$, we use an estimate of evolutionary fitness based on a pre-trained protein language model (PLM),
\begin{equation} \label{eqn:plm_f}
    \f(x) \triangleq \left(\log \frac{\pr_\textsc{plm}(x \mid x^{wt}_{1:\tau})}{\pr_\textsc{plm}(x^{wt}_{\tau+1:\tau+L} \mid x^{wt}_{1:\tau})} + 1\right)_+
\end{equation}
where we set $\pr_\textsc{plm}(x)$ to ProGen2's likelihood and $x^{wt}$ to a wild-type protein sequence \citep{nijkamp_progen2_2023}. The likelihood ratio is the \textit{evolutionary index}, which compares the probability of evolution generating $x$ versus a natural sequence; it has been shown to reliably correlate with diverse measures of protein activity \citep{hopf_mutation_2017,riesselman_deep_2018}. The conditional $\pr_\textsc{plm}(x \mid x^{wt}_{1:\tau})$ is predictive of whether a peptide $x$ will interact with $x^{wt}_{1:\tau}$ \citep{chen_target_2026}. Sequences below a threshold evolutionary index are typically dysfunctional or pathogenic, so we apply a cutoff \citep{frazer_disease_2021}.
Unlike benchmarks based on raw experimental data, $\pr_\textsc{plm}(x)$ is defined over all of sequence space, allowing us to compute the activity of a mixture \citep{notin_proteingym_2023}.
Crucially, we can approximate the integral $\langle \f, q \rangle = \Eb_{x \sim q(x)}[\f(x)]$ even when active sequences are rare under $\f$, by generating samples from the same PLM whose likelihood defines $\f$.
We use importance sampling with a defensive mixture as the proposal, $s(x) = \frac{1}{2} q(x) + \frac{1}{2}\pr_\textsc{plm}(x \mid x^{wt}_{1:\tau})$, ensuring we evaluate $q(x)\f(x)$ in regions where either $q$ or $\f$ are high   \citep{hesterberg_weighted_1995} (\Cref{app:IS}).

We consider three wild-type proteins $x^{wt}$: the CO$_2$-anhydrase SazCA, the CO$_2$-capturing RuBisCO, and the plastic-degrading PETase.
We prompt with the first $\tau=95$ amino acids of the wild-type, and design the next 5, so the molecular space $\Xc = \Bc^5$ is size $d=|\Xc|=3{,}200{,}000$ molecules.
The LIDS model is given an unbiased prior $\pi(\theta)$, with no evolutionary or physicochemical data that could be predictive of $\pr_\textsc{plm}(x)$.
LOCK uses BLOSUM, a model of amino acid correlations based on evolutionary data \citep{jankowiak_flexible_2026,Henikoff1992-ks}.
At every experimental round, we test a batch of $E=96$ designs.
We run 5 repeats, with independent initialization and measurement noise with $\textsc{PSNR}=80$dB.
We run the individual synthesis methods for 100 rounds and, due to computational expense, LIDS for 5. Details in \Cref{app:fplm}

\begin{figure}
  \begin{subfigure}{0.33\textwidth}
    \includegraphics[width=\linewidth]{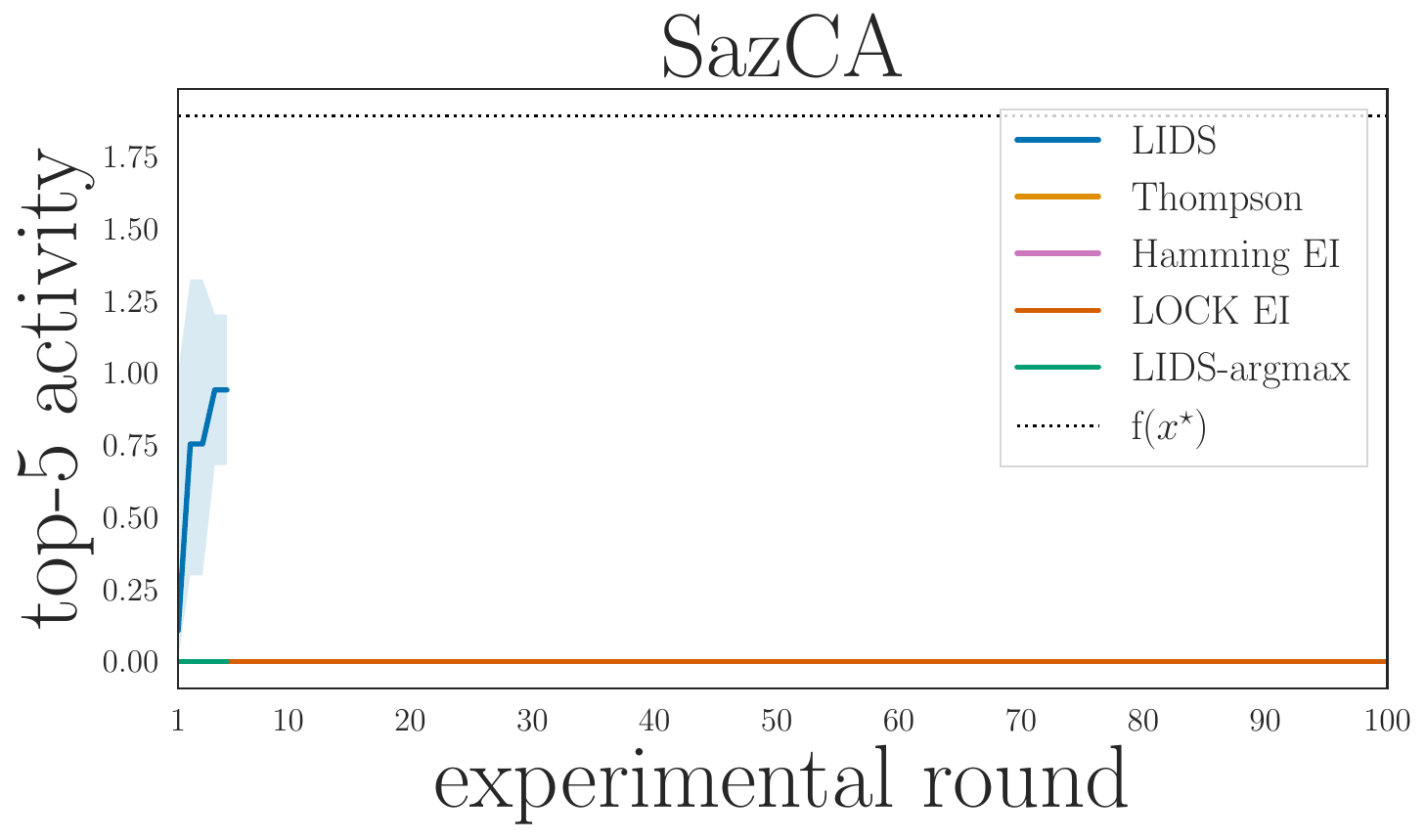}
    \caption{}
    \label{fig:AA_PETase}
  \end{subfigure}
  \hfill
  \begin{subfigure}{0.33\textwidth}
    \includegraphics[width=\linewidth]{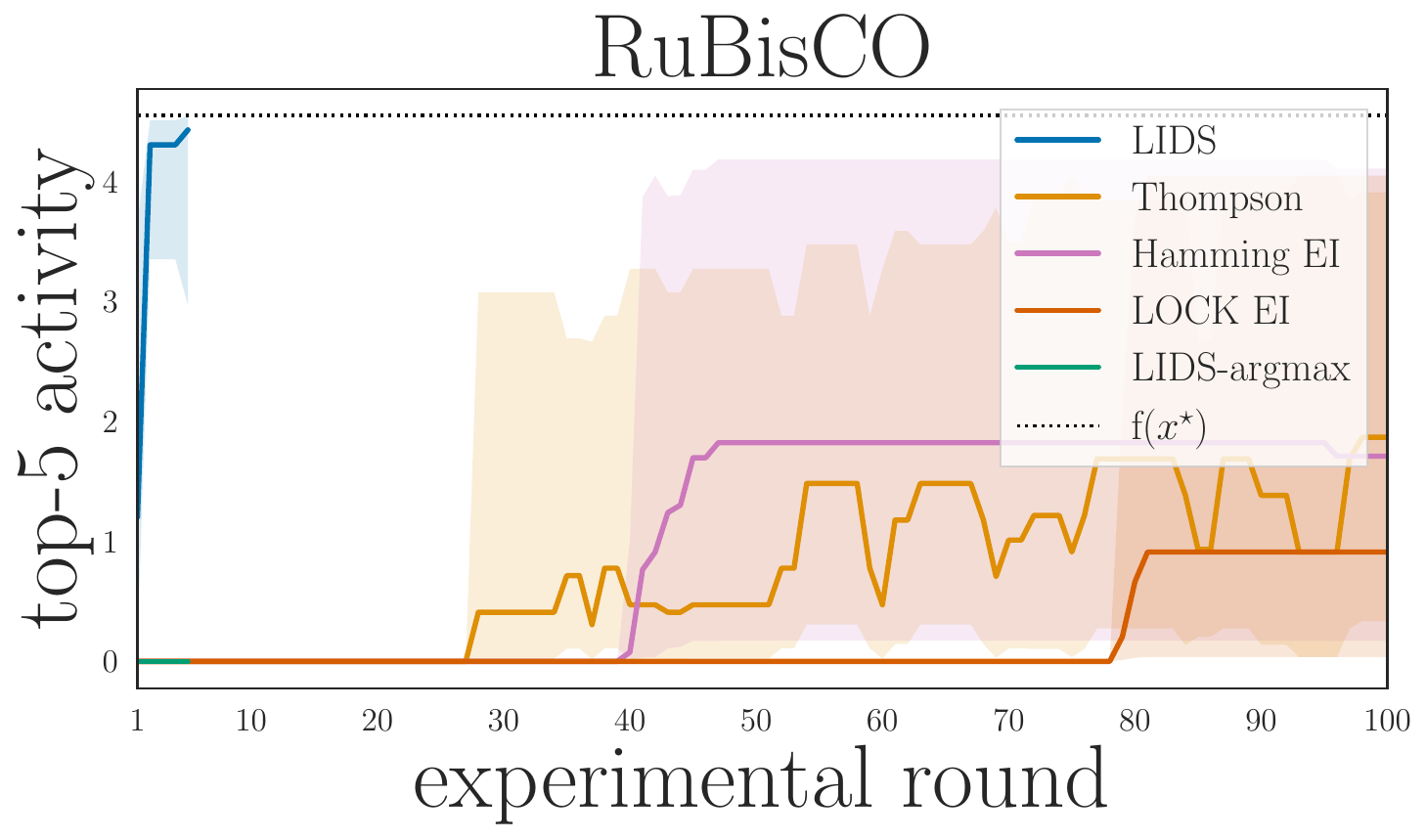}
    \caption{}
    \label{fig:AA_rubisco}
  \end{subfigure}\hfill
  \begin{subfigure}{0.33\textwidth}
    \includegraphics[width=\linewidth]{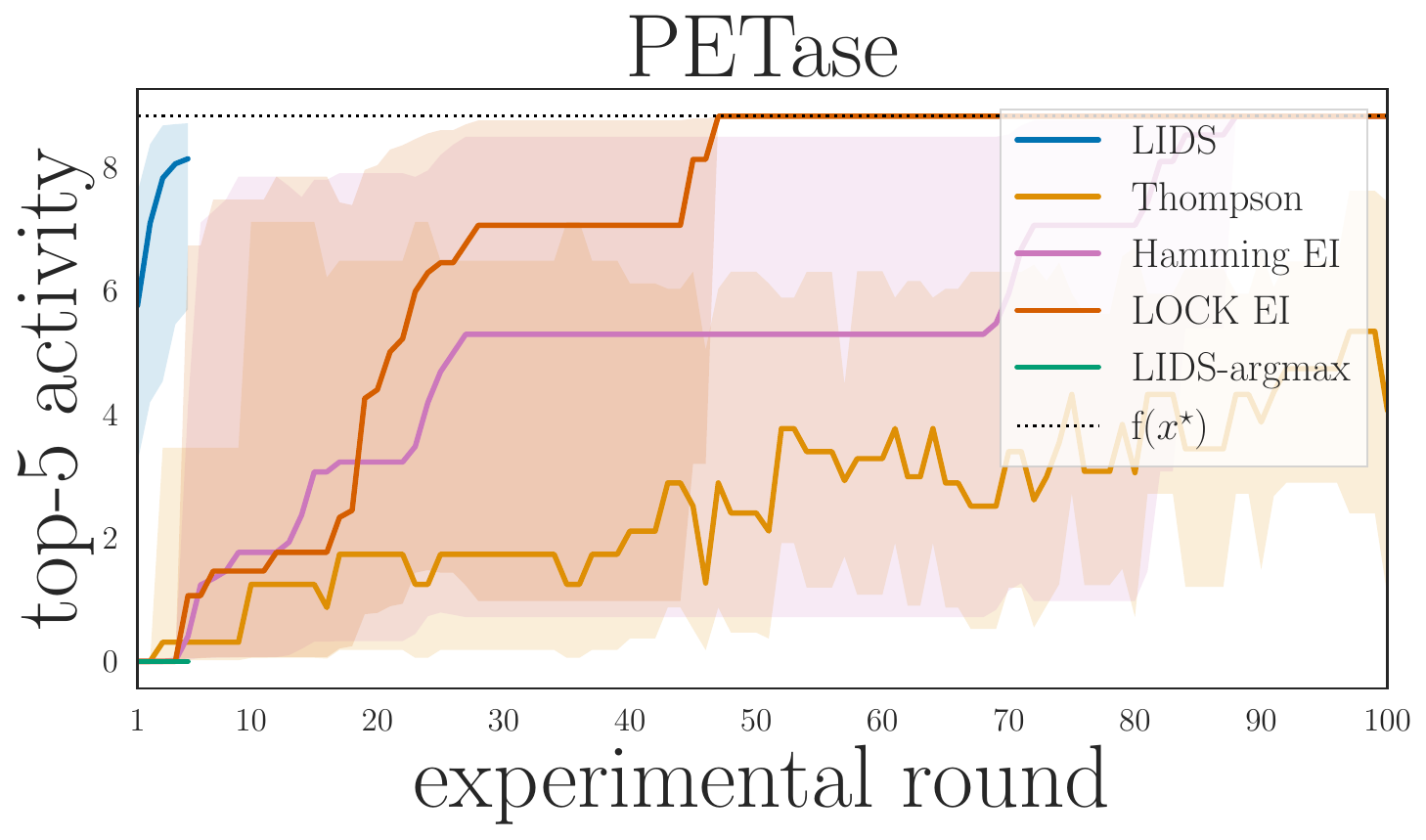}
    \caption{}
    \label{fig:AA_PETase}
  \end{subfigure}
\vspace{-2em}
  \caption{\textbf{PLM oracle.} LIDS, Thompson, Hamming EI, and LOCK EI, applied to a PLM oracle, with low noise ($\textsc{PSNR}=80$dB). The maximum activity of the 5 latest test molecules $\max \f(\hat{x}^\star_{t-5}),...,\f(\hat{x}^\star_{t})$, is plotted as a function of the experimental round $t$, averaged over 5 duplicates. The 95\% CI is calculated using a logit transform and the delta method, so that it stays bounded in the interval $[0,\f(x^\star)]$ (See \Cref{app:CI}). }
  \label{fig:AA_empirical}
\vspace{-1em}
\end{figure}

\paragraph{Results} LIDS finds high-activity molecules faster than methods based on individual synthesis (\Cref{fig:AA_empirical}).
LIDS' improvement is smallest on the PETase landscape, where active sequences are more common: here LIDS finds the global optimum in 3 rounds (median across repeats), whereas the GP methods take 27 rounds (median) (\Cref{tab:rounds_sparse_aa}).
On the RuBisCO landscape, LIDS finds the global optimum in 2 rounds (median), whereas individual synthesis methods did not find it within 100 rounds, in the majority of repeats (\Cref{tab:rounds_sparse_aa}).
We kept running Hamming EI for an additional 100 rounds; even so, it only found the optimum in 3/5 repeats.
On the SazCA landscape, LIDS does not find the unique global optimum, but identifies sequences with non-zero activity within 5 rounds, while the alternatives find nothing after 100 rounds (\Cref{fig:AA_empirical}a).
When LIDS failed to find the global optimum, it appeared to be due to inefficient posterior sampling. 
Reducing the number of posterior samples $M$ degraded LIDS' performance (\Cref{fig:M-ablation}).

\begin{wrapfigure}{r}{0.35\textwidth}
    \vspace{-2em}
    \centering
    \includegraphics[width=\linewidth]{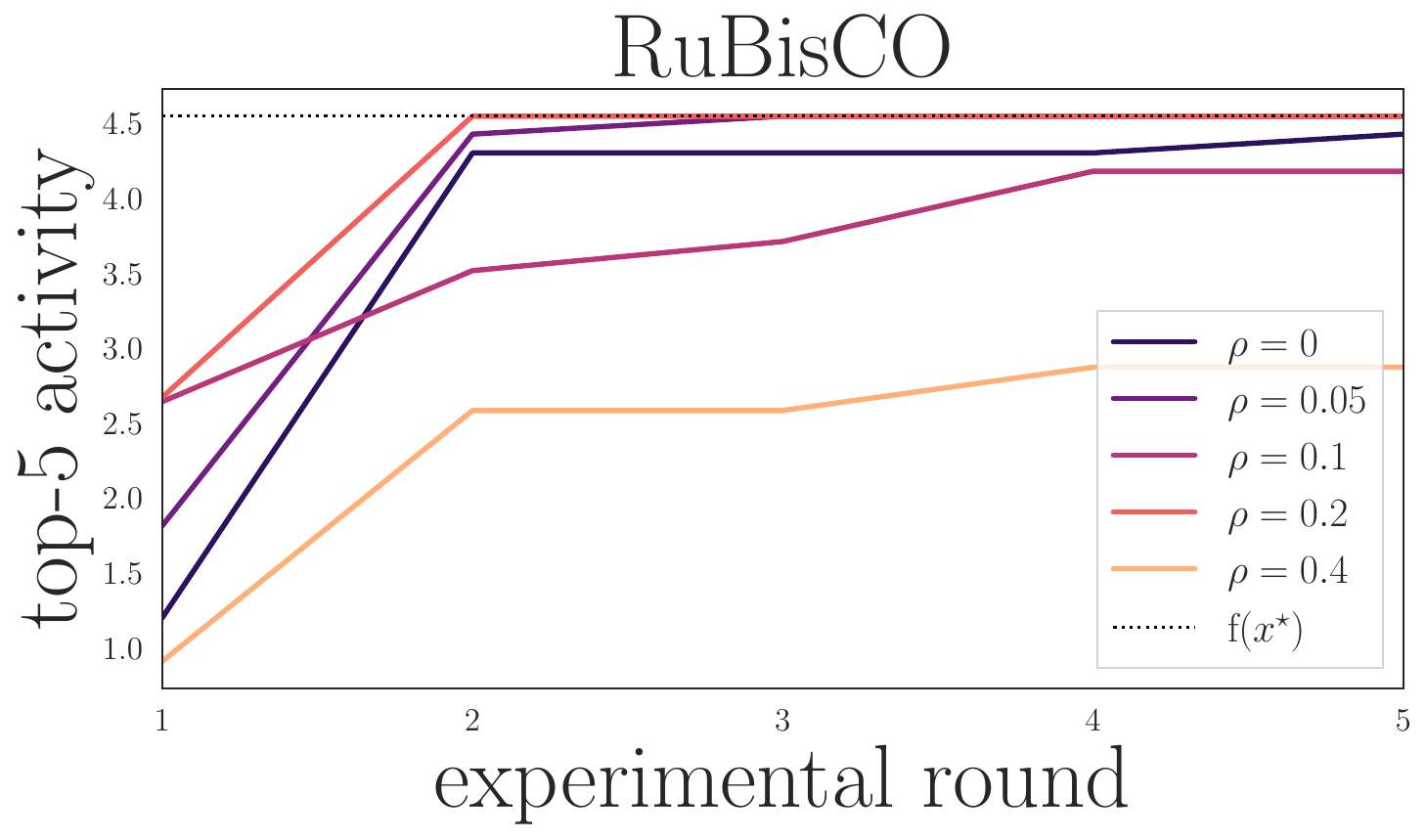}
    \vspace{-2em}
    \caption{\textbf{LIDS with synthesis errors.} We add random errors to the designed mixtures $q_t$ by replacing $\beta_{sl}$ with $(1-\rho)\beta_{sl} + \rho\omega$ where $\omega \sim \mathrm{Dir}(\vec{1}_B)$.}
    \label{fig:rho-ablation}
    \vspace{-3em}
\end{wrapfigure}
LIDS' advantage is from information-dense stochastic synthesis. 
Ablating stochastic synthesis, by replacing each design $q_t$ with its optimum, ablated performance entirely (\textit{LIDS-argmax}).
So did ablating stochastic synthesis by using Thompson sampling with the same model $\f_\theta$ and prior $\pi(\theta)$ as LIDS (\textit{Thompson}).
Using a SOTA GP model, containing additional prior information about $\f$, did not compensate for the advantages of stochastic synthesis (\textit{LOCK}). 

LIDS is robust: its model $\f_\theta$ is misspecified for these oracle $\f$.
Adding synthesis errors, so that the tested $q_t$ differs from the design, began substantially reducing performance only when 40\% of every mixture $\beta_{sl}$ was corrupted (\Cref{fig:rho-ablation}).

\section{Discussion}

We showed that using information-dense stochastic synthesis, instead of synthesizing and testing individual molecules, can lead to very large speedups on hard molecular discovery problems.

\paragraph{Limitations.}
First, LIDS has only been evaluated in simulation, not in the wetlab.
Second, LIDS shifts costs from experimentation to computation: it requires fewer wetlab experiments but much more compute than individual synthesis.
The challenge is LIDS requires (a) calculating the predicted activity of a complex mixture rather than a single molecule, (b) optimizing the EIG, and (c) large numbers of posterior samples.
Scaling LIDS to larger molecules and more complex surrogate models $\f_\theta$ will require accelerating Bayesian experimental design, e.g. using amortization \citep{foster_unified_2020,foster_deep_2021,Hedman2025-zs}.
The mixture's activity $\langle \f_\theta, q\rangle$ could remain tractable if the surrogate $\f_\theta$ is a smooth and decomposable probabilistic circuit \citep{vergari_compositional_2021}.

\paragraph{Outlook.}
Scaling molecular machine learning requires scaling laboratory feedback.
One axis for scaling is performing more experiments. When testing individual molecules, this is essentially the only hope, since no adaptive algorithm can overcome a sparse activity map \citep[\Cref{thm:indiv-limit}]{Bach2020-yw}.
LIDS opens another, complementary approach to scaling: achieving precise control over synthesis, higher synthesis yields, and higher assay signal-to-noise ratios.
These improvements do not change the number of datapoints, but rather the information they carry.
This opens new opportunities and directions for scaling automated molecular science.

\subsubsection*{Acknowledgments}
We wish to thank Victor Veitch, Katrine Qvortrup, Elizabeth Wood, and Alan Amin for useful discussions and suggestions. 
This work was supported by a Start Package Grant from the Novo Nordisk Foundation to ENW.
We thank the DTU Computing Center and the Pioneer Centre for AI for providing compute resources and support.

\bibliography{references,references-enw}
\bibliographystyle{iclr2027_conference}

\appendix
\crefalias{section}{appendix}  \crefalias{subsection}{appendix}

\section{Notation}\label{app:notation}

\begin{table}[H]
\begin{tabular}{ll}
\hline
Notation                                                                                    & Description                                                                                           \\ \hline
$\Xc$                                                                                       & The molecular space, usually a Hamming space                                                          \\ \hline
$\Pc(\Xc)$                                                                                  & \makecell[l]{$=\{\pr(\cdot) : \pr(x) \geq 0, \int_\Xc\pr(x) \di x = 1\}$                              \\
                                                                                            The set of probability distributions on $\Xc$}                                                          \\ \hline
$\f: \Xc \to \Rb_+$                                                                         & The molecule-activity map                                                                             \\ \hline
$q: \Xc \to \Rb_+$                                                                          & \makecell[l]{The design - can be a mixture or a delta.                                                \\
                                                                                            The restriction on $q$ comes from the synthesis model.}                                                 \\ \hline
$\bar{q}= \frac{q}{c} \in \Pc(\Xc)$                                                         & \makecell[l]{The normalized design, where $c$ is the yield,                                           \\
                                                                                            the total number of molecules produced by a synthesis.}                                                 \\ \hline
$c \in \Rb_+$                                                                               & \makecell[l]{The chemical yield. The total number of molecules                                        \\
                                                                                            produced by the synthesis model.}                                                                       \\ \hline
$\epsilon \sim \mathrm{Normal}(0,\sigma)$                                                   & The assay noise, drawn from a Gaussian.                                                               \\ \hline
$\Ib[\cdot]$                                                                                & \makecell[l]{The indicator function, which takes                                                      \\
                                                                                            value 1 when the input is true, and 0 otherwise.}                                                       \\ \hline
$\unif(z)$                                                                                  & A uniform distribution over a set $z$                                                                 \\ \hline
$[d] \triangleq \{1, \ldots, d\}$                                                           & A set of d elements                                                                                   \\ \hline
$\Delta_{B-1} = \{v \in \Rb^B_+: \sum_j v_j = 1\}$                                          & The simplex.                                                                                          \\ \hline
$\vec{1}_d$                                                                                 & The $d$-vector of all ones                                                                            \\ \hline
$\mathrm{LSE}_{k=1}^K[x_k] = \log \sum_{k=1}^K\exp[x_k]$                                    & The 'logsumexp' operation.                                                                            \\ \hline
\end{tabular}
\end{table}

\section{Proof of \Cref{thm:indiv-limit}} \label{apx:proof-indiv-limit}

Assume $\f(x) = \Ib(x = x^\star)$ where $x^\star \sim \unif(1, \ldots, d)$, and there is no noise in the assay, so we measure $y_t = \f(x_t)$. Define expected regret, $R \triangleq \Eb_{x^\star}[\Eb_\pi[\sum_{t=1}^\infty \Ib(X_t \neq X^\star)]]$. 
Let $\pi_0(x_t \mid \Dc_{t-1})$ denote a policy that selects the next candidate uniformly at random from among the remaining candidates until the hit is discovered.
\begin{equation}
    \pi_0(x_t \mid \Dc_{t-1}) =
    \begin{cases}
    x_t \sim \unif(\{x \notin \Dc_{t-1}\}) &\text{ if } x^\star \notin \Dc_{t-1}\\
        x_t = x^\star &\text{ otherwise }
    \end{cases} 
\end{equation}
This policy has expected regret $(d-1)/2$. There does not exist a policy $\pi$ with lower expected regret.

\begin{proof}
We first calculate the regret of the randomized policy $\pi_0$. We have
\begin{align}
    R_0 &= \Eb_{x^\star}[\Eb_\pi[\sum_{t=1}^\infty \Ib(X_t \neq X^\star)]]\\ &= \sum_{t=1}^\infty \pr(x_t \neq x^\star) \\
    &=\pr(x_1 \neq x^\star) + \pr(x_1 \neq x^\star)\pr(x_2 \neq x^\star \mid x_1 \neq x^\star) + \pr(x_{1:2} \neq x^\star) \pr(x_3 \neq x^\star \mid x_{1:2} \neq x^\star) + \cdots\\
    &= \frac{d-1}{d} + \frac{d-1}{d} \frac{d-2}{d-1} + \frac{d-2}{d}\frac{d-3}{d-2} + \cdots\\\
    &= \frac{d-1}{d} + \frac{d-2}{d} + \cdots + \frac{1}{d}\\ 
    &= \frac{d-1}{2}
\end{align}

Next consider an arbitrary policy $\pi$. We will show that it cannot achieve lower regret than the randomized policy. 
We assume without loss of generality that $\pi$ plays $x_t = x^\star$ once $x^\star \in \Dc$, since if it does not, its regret will only be improved by substituting this rule.
Now start with the regret at $t=1$. We have
\begin{equation}
    \Eb_{x^\star} [\Eb_{\pi}[\Ib[X_1 \neq X^\star]]] = \frac{1}{d} \sum_{x^\star=1}^d \pi(x_1 \neq x^\star) = \frac{1}{d} \sum_{x^\star=1}^d [1 - \pi(x_1 = x^\star)] = \frac{1}{d}(d - 1)
\end{equation}
Next consider the regret at $t=2$. We have
\begin{align}
    \Eb_{x^\star}[\Eb_\pi[\Ib[X_{1:2} \neq X^\star]]] &= 
    \pr(x_1 \neq x^\star) \Eb_{\pr(x_1 \mid x_1 \neq x^\star) \pr(x^\star \mid x_1 \neq x^\star, x_1) \pr(x_2 \mid x_1, x^\star, x_1 \neq x^\star)}[\Ib(x_2 \neq x^\star)]\\
    &= \pr(x_1 \neq x^\star) \Eb_{\pr(x_1 \mid x_1 \neq x^\star) \pr(x^\star \mid x_1 \neq x^\star, x_1)}[\pi(x_2 \neq x^\star \mid x_1, y_1 = 0)]\\
    &= \pr(x_1 \neq x^\star) \Eb_{\pr(x_1 \mid x_1 \neq x^\star)}[ \frac{1}{d-1} \sum_{x^\star \neq x_1} \pi(x_2 \neq x^\star \mid x_1, y_1 = 0)]\\
    &= \pr(x_1 \neq x^\star) \Eb_{\pr(x_1 \mid x_1 \neq x^\star)}[1 - \frac{1}{d-1} \sum_{x^\star \neq x_1} \pi(x_2 = x^\star \mid x_1, y_1=0)]\\
    &\ge \pr(x_1 \neq x^\star) \Eb_{\pr(x_1 \mid x_1 \neq x^\star)}[1 - \frac{1}{d-1}]\\
    &= \frac{d-1}{d} \frac{d-2}{d-1} = \frac{d-2}{d}
\end{align}
Where for the second line we have used the fact that the distribution of $x_2$ just depends on the policy, which just depends on $x_1, y_1$ and not on the unobserved $x^\star$. The inequality is strict when $\pi$ does not assign any probability to the previous observation $x_1$ when $y_1 = 0$.

Repeating this argument for arbitrary $t$,
\begin{align}
    \Eb_{x^\star}[\Eb_\pi[\Ib[X_{1:t} \neq X^\star]]] &= \pr(x_{1:t-1} \neq x^\star) \Eb_{\pr(x_{1:t-1} \mid x_{1:t-1} \neq x^\star) \pr(x^\star \mid x_{1:t-1} \neq x^\star, x_{1:t-1})}[\pi(x_t \neq x^\star \mid x_{1:t-1}, y_{1:t-1}=0)]\\
    &= \pr(x_{1:t-1} \neq x^\star) \Eb_{\pr(x_{1:t-1} \mid x_{1:t-1} \neq x^\star)} \frac{1}{d - t+1} \sum_{x^\star \neq x_{1:t-1}} \pi(x_t \neq x^\star \mid x_{1:t-1}, y_{1:t-1}=0)\\
    &\ge \pr(x_{1:t-1} \neq x^\star) \frac{d-t}{d-t+1}\\
    &= \frac{d-t}{d}
\end{align}
where the final step is from induction over $t$.

Summing over all $t$, we obtain the result.

\end{proof}

\section{Popular acquisition functions are not useful to optimize mixtures} \label{app:acquisition_functions}

Most acquisition functions are useless when it comes to designing mixtures. The proof for the Upper Confidence Bound and Expected Improvement acquisitions functions are given below, which hinge on the fact that while $\f: \Xc \to \Rb$ is an arbitrary bounded function, in a lifted setting $\langle \f, q \rangle$ is an affine function over the convex set $q \in \Pc(\Xc)$. 

Define the set of delta functions as $\Sc \triangleq \{\delta_x :x\in \Xc\} \subset \Pc(\Xc)$

\paragraph{Upper Confidence Bound}

The UCB acquisition function for a mixture is defined as \citep{auer_finite-time_2002, abbasi-yadkori_improved_2011}
\begin{equation}
    \mathrm{UCB}(q) \triangleq \sup_{\f \in \Fc_{t-1}} \langle \f, q \rangle,
\end{equation}
where $\Fc_{t-1}$ is a confidence set of possible $\f$, conditioned on the data at timepoint $t-1$. Note when $q \in \Sc$ this reduces to $\mathrm{UCB}(x) = \sup_{\f \in \Fc_{t-1}} \f (x)$ \citep{abbasi-yadkori_improved_2011}. As this is the pointwise supremum over linear functions in $q$, $\mathrm{UCB(q)}$ is convex in $q$. The next mixture is chosen by $q_t = \amax_{q \in \Pc(\Xc)} \mathrm{UCB}(q)$. Since $\Sc \subset \Pc(\Xc)$ we have the inequality $\max_{q \in \Sc} \mathrm{UCB}(q) \leq \max_{q \in \Pc(\Xc)} \mathrm{UCB}(q)$. A second inequality follows by Jensen's inequality:
\begin{align}
    \mathrm{UCB}\left (\int_\Xc q(x)\cdot \delta_x \di x \right ) &\leq \int_\Xc q(x) \mathrm{UCB}(\delta_x) \di x \\
    &\leq \max_{x \in \Xc} \mathrm{UCB}(\delta_x) = \max_{q \in \Sc} \mathrm{UCB}(q)
\end{align}
This establishes that $\max_{q \in \Sc} \mathrm{UCB}(q) = \max_{q \in \Pc(\Xc)} \mathrm{UCB}(q)$, and as such there must always exist a $q^\star \in \Sc$ which is a maximizer of $\mathrm{UCB}(q)$.

\paragraph{Expected Improvement}

Let $y = \langle \f, q\rangle$ denote the experimental outcome, and $y^\star = \underset{i < t}{\max}\, y_i$ denote the maximum experimental outcome so far. The improvement is defined as $I(q) \triangleq \max (0, \langle \f, q \rangle - y^\star)$, and the expected improvement is the expectation over the randomness in $\f$,

\begin{equation}
    \mathrm{EI}(q) = \Eb_{\pr(\f \mid \Dc_{t-1})} [\max (0, \langle \f, q \rangle - y^\star)].
\end{equation}

Inside the expectation is the max of two affine functions, $\max(0, \langle \f, q\rangle - y^\star)$. Since the max of two affine functions is a convex function, $\mathrm{EI}(q)$ is an expectation over terms that are convex in $q$, which is itself, a convex function. 
It then follows from the same logic as for the $\mathrm{UCB}$ acquisition function, that $\max_{q \in \Sc} \mathrm{EI}(q) \leq \max_{q \in \Pc(\Xc)} \mathrm{EI}(q)$, and
\begin{align}
    \mathrm{EI}\left (\int_\Xc q(x)\cdot \delta_x \di x \right ) &\leq \int_\Xc q(x) \mathrm{EI}(\delta_x) \di x \\
    &\leq \max_{x \in \Xc} \mathrm{EI}(\delta_x) = \max_{q \in \Sc} \mathrm{EI}(q),
\end{align}
and as such there must always exist a $q^\star \in \Sc$ which is a maximizer of $\mathrm{EI}(q)$.

\section{Derivation of variational lower bound on EIG}\label{app:variational_lower_bound}

We can construct a lower variational bound on the mutual information, $\Ic_{BA}$, where $r_\phi$ denotes a variational distribution with parameters $\phi$, and $\KL(\pr\| r) \geq 0$ denotes the KL divergence between distributions $\pr$ and $r$ \citep{barber_information_2003}.
To rewrite the expectation over $y_t$, we use the fact that $y_t$ is a deterministic function of $q_t$, $\epsilon_t$ and $\theta$, using $y_t = \langle \f_\theta, q_t \rangle + \sigma \epsilon_t$ for some noise level $\sigma$ and $\epsilon_t \sim \mathrm{Normal}(0,1)$.

\begin{align}
	EIG_{x^\star}(q_t) &= \Hc(\pr(x^\star_\theta \mid \Dc_{t-1})) + \Eb_{\pr(\theta \mid \Dc_{t-1})\pr(y_t \mid q_t, \theta)}[\log \pr(x^\star_\theta \mid \{y_t, q_t\})]\\
	&= \Hc(\pr(x^\star_\theta \mid \Dc_{t-1})) + \Eb_{\pr(\theta \mid \Dc_{t-1})\pr(y_t \mid q_t, \theta)}[\log r_\phi(x^\star_\theta \mid y_t) + \KL(\pr(\cdot \mid \{y_t, q_t\})\|r_\phi(\cdot \mid y_t))]\\
	&\geq \Hc(\pr(x^\star_\theta \mid \Dc_{t-1})) + \Eb_{\pr(\theta \mid \Dc_{t-1})\pr(y_t \mid q_t, \theta)}[\log r_\phi(x^\star_\theta \mid y_t)]\\	
	&= \Hc(\pr(x^\star_\theta \mid \Dc_{t-1})) + \Eb_{\pr(\theta \mid \Dc_{t-1})\pr(\epsilon_t)}[\log r_\phi(x^\star_\theta \mid \langle \f_\theta, q_t \rangle + \sigma \epsilon_t)]\\
	&\triangleq \Ic_{BA}(q_t, \phi)
\end{align}

\section{Details on variational distribution $r_\phi$} \label{app:variational_dist}

\subsection{Architecture}

The variational distribution approximates the conditional $\pr(x^\star \mid y, q) \approx r_\phi(x^\star \mid y)$. As the design $q=q_\psi$ is fixed for any step of training, we don't include it as an input to the variational distribution. The effect of $\psi$ on the EIG enters through the likelihood $\pr(y \mid \langle \f_\theta, q_\psi\rangle)$.
For a batch of experimental outcomes $\bar{y} = [y_t,...y_{t+E}]$ we parameterize the variational distribution $r_\phi: \Rb^E \to \Pc(\Xc)$ as a linear regression from $\bar{y}$ to a position-wise independent categorical distribution $\cat(x \mid \beta) = \prod_{j=1}^L \beta_{jx_{j}}$, by

\begin{equation}
    r_{\phi = (A,C)}(x^\star \mid \bar{y}) = \cat(x^\star \mid \mathrm{softmax} (A \cdot \bar{y} + C)),
\end{equation}
where $A \in \Rb^{L\times B\times E}$ is a weight tensor, and $C\in\Rb^{L\times B}$ is a bias.
The $\mathrm{softmax}$ is applied over the last dimension, which has size $B$.

\subsection{Training}\label{app:variational_dist_training}

For Gaussian noise, using the reparameterization trick to reduce variance \citep{foster_unified_2020}, we obtain the gradient
\begin{equation}
    \nabla_{\psi,\phi} \Ic_{BA}(x^\star ; \psi, \phi) \approx \frac{1}{K} \sum_{j=1}^M  \nabla_{\psi,\phi}\log r_\phi(x^\star_{\theta_j} \mid \langle \f_{\theta_j}, q_\psi \rangle + \sigma \epsilon_j)
\end{equation}
where $\epsilon_{1:M} \overset{iid}{\sim} \mathrm{Normal}(0, 1)$.

The variational distribution parameters $\phi$, as well as mixture parameters $\psi$, are optimized jointly using stochastic gradient ascent on $\nabla_{\psi,\phi} \Ic_{BA}(x^\star ; \psi, \phi)$, using the AdamW optimizer with a learning rate of $0.01$, momentum of $0.9$, and weight decay of $10^{-5}$ \citep{loshchilov_decoupled_2018}. We train all our models on a dataset of 100000 samples of $\theta$, either drawn from the prior $\pi(\theta)$ (see \Cref{app:bnn-prior}), or from the approximate posterior, $\pr(\theta \mid \Dc)$ (see \Cref{app:posterior_approx}). We use 10000 linear warmup steps, and train for either 1000 epochs (for the sparse oracle, \Cref{sec:res_sparse}), or 5000 epoch (for the PLM oracle, \Cref{sec:res_plm}). Training is done on an A100 Nvidia GPU with a batchsize of 512.

\section{Details on Bayesian neural network $\f_\theta$} \label{app:bnn}

We consider a Bayesian neural network $\f_{\theta = (W,V)}(x) = \sum_{n=1}^N W_n \exp(V_n \cdot x)$ as a function class for sparse molecule-activity maps.

\subsection{Sparsity-inducing prior}\label{app:bnn-prior}

We define $w\cdot u_n \triangleq W_n$, and $\log v_n \triangleq V_n$.
We restrict the domains of $u$ and $v$ to be $u \in \Uc = \{u\in [0,1]^N:\|u\|_\infty=1\}$ and $v \in \Vc = \{v\in [0,1]^{N\times L \times B}:\|v_{n,l,\cdot}\|_\infty=1 \ \forall \ n,l\}$.
We can then construct a prior, $\pi(\theta) = \pi(w,u,v)$ that induces sparsity in the filters $V_1,...V_N$ by setting

\begin{align}
    v'_{nlb} &\sim \mathrm{Expon}(\frac{1}{2}) \\
    v_{nlb} &= \exp(v'_{nlb} - \max_b v'_{nlb})
\end{align}
and 
\begin{align}
    \log u'_n &\sim \Nc(0,1) \\
    u_n &= \exp(\log u'_n - \max_n(\log u'_n))
\end{align}

Where $\mathrm{Expon(\gamma)}$ denotes the exponential distribution with rate parameter $\gamma$. 
This choice is motivated by the use of the Laplace distribution for Bayesian LASSO \citep{park_bayesian_2008}.
It carries the belief that only few $v_{nlb}$ are relevant. 
There exists extensive literature on sparsity-inducing priors \citep{carvalho_handling_2009, piironen_sparsity_2017, andersen_bayesian_2014, betancourt_sparsity_2021}. 
We found the process above to be robust to posterior approximation by NUTS, whereas many other priors led to difficult posterior geometry, a significant amount of divergences, and high Gelman-Rubin statistics.
The function $h(x_n) = \exp(x'_n - \max_n x'_n)$ is used to bring $u$ and $v$ to their restricted domains, $u \in \Uc$ and $v \in \Vc$.
Our prior on $w$ carries our belief about the maximum activity that the assay can record, $\|\f\|_\infty \geq w$, with equality if the filters are disjoint $\prod_l \langle v_{nl}, v_{n'l}\rangle = 0 \forall n \neq n'$.
Additionally, when the filters are disjoint, at least one $x^\star$ will have activity $\f_\theta (x^\star) = w$, and all other $x' \neq x^\star$ will have lower activity $\f_\theta(x') \leq w$. 
This makes it easy to incorporate prior information about the maximum activity, by setting $\pi(\|\f\|_\infty) = \pi(w)$. 
In the our simulations we use $\pr(w) = \delta_{w=1}$ (for \Cref{sec:res_sparse}) and $\pr(w) = \mathrm{Gamma}(w-1\mid 1,1)$ (for \Cref{sec:res_plm}), where $\mathrm{Gamma(x \mid \alpha,\beta)}$ is the gamma pdf with shape $\alpha$ and rate $\beta$.

\subsection{Regularization}\label{app:bnn-regularization}
We add a regularization term to the joint distribution, to prevent mode collapse of the different bump functions.
Additionally, this encourages $\f_\theta$ for which the argmax approximation is tighter, since it encourages the filters $V_1,...V_N$ to be disjoint.
This term is defined as 
\begin{equation}
	\log p'(\theta, \Dc) = \log p(\theta, \Dc) - \lambda \cdot \frac{1}{B^L} \sum_{n < m} \prod_{l=1}^{L} \langle v_{n,l}, v_{m,l} \rangle,
\end{equation}
where the normalization is to bring the regularization term to the interval $[0, N_{pairs}]$, and $\lambda$ is a hyperparameter that controls the strength of the regularization. $N_{pairs} = \frac{N(N-1)}{2}$ is the number of pairs of bump functions.

\subsection{Inner product $\langle \f_\theta, q_\psi \rangle$} \label{app:bnn-inner}

The inner product can be computed by
\begin{equation}
	\langle \f_\theta, q_\psi\rangle = c\sum_{s=1}^S \alpha_{s} \sum_{n=1}^N W_n \prod_{l=1}^{L} \langle \exp{V_{nl}}, \beta_{sl} \rangle
\end{equation}
The synthesis model is defined as $\bar{q}_{\psi=(\alpha, \beta)}(x) = \sum_{s=1}^S \alpha_s\prod_{l=1}^L \langle \beta_{sl}, x_l \rangle$.
To see how the inner product $\langle \f_\theta, \bar{q}_\psi \rangle$ can be computed in closed form, first rewrite the network as

\begin{align}
    \f_\theta(x) &= w \cdot \sum_{n=1}^N u_n  \exp{ \sum_{l=1}^L \sum_{b=1}^B \log v_{nlb} \cdot x_{lb}} \\
    &= w \cdot \sum_{n=1}^N u_n  \exp{ \sum_{l=1}^L \log \sum_{b=1}^B v_{nlb} \cdot x_{lb}} \\
    &= w \cdot \sum_{n=1}^N u_n  \prod_{l=1}^L \sum_{b=1}^B v_{nlb} \cdot x_{lb} \\
    &= w \cdot \sum_{n=1}^N u_n  \prod_{l=1}^L \langle v_{nl} \cdot x_{l} \rangle.
\end{align}

The inner product can then be derived as

\begin{align}
    \langle \f_\theta, \bar{q}_\psi \rangle &=\sum_{x \in \Xc} \f_\theta(x)\bar q_\psi(x) \\
    &= \sum_{x \in \Xc} \left(w\cdot \sum_{n=1}^N u_n \prod_{l=1}^L \langle v_{nl}, x_l \rangle \right) \cdot\left(\sum_{s=1}^S \alpha_s \prod_{l=1}^L \langle \beta_{sl}, x_l \rangle \right) \\
    &= \sum_{s=1}^S \alpha_s w \sum_{n=1}^N u_n  \sum_{x \in \Xc} \prod_{l=1}^L \langle \beta_{sl}, x_l \rangle \langle v_{nl}, x_l \rangle \\
    &= \sum_{s=1}^S \alpha_s w \sum_{n=1}^N u_n \prod_{l=1}^L \langle v_{nl}, \beta_{sl} \rangle \\
\end{align}

In practice, for numerical stability, we compute the functional inner product in log-space by

\begin{align}
    \log \langle \f_\theta, \bar{q}_\psi \rangle = \overunderset{S}{s=1}{\mathrm{LSE}}\left[ \log\alpha_s + \log w + \underset{N}{\mathrm{LSE}}\left[ \log u_n + \sum_B\underset{B}{\mathrm{LSE}}\left[\log v_{nlb} + \log\beta_{slb}\right]\right]\right],
\end{align}
where $\mathrm{LSE}_{k=1}^K[x_k] = \log \sum_{k=1}^K\exp[x_k]$ is the 'logsumexp' operation

\subsection{Approximation of $x^\star$}\label{app:bnn-x-star-approximation}

This function class has a tractable approximation of $x^\star = \amax_x \f_\theta(x)$, or equivalently, $\amax_{x} \log \f_\theta(x)$. First, observe that we can write
\begin{align}
	\log \f_\theta(x) =  \log \sum_n \exp(\log w_n + \sum_\ell \log v_{n,\ell}^\top x_\ell)\\
	= \overunderset{N}{n=1}{\mathrm{LSE}} \left[\log w_n + \sum_\ell \log v_{n,\ell}^\top x_\ell\right].
\end{align}
The $\mathrm{LSE}$ operation can be bounded by
\begin{equation}
	\max_n \log w_n + \sum_\ell \log v_{n,\ell}^\top x_\ell \le \overunderset{N}{n=1}{\mathrm{LSE}} \left[\log w_n + \sum_\ell \log v_{n,\ell}^\top x_\ell\right] \le \max_n \log w_n + \sum_\ell \log v_{n,\ell}^\top x_\ell + \log N.
\end{equation}
To approximate the maximum over $\f_\theta(x)$, we can maximize the lower bound. We have
\begin{equation}
	\hat x = \underset{x}{\amax} \max_n \log w_n + \sum_\ell \log v_{n,\ell}^\top x_\ell
\end{equation}
This is tractable. Define
\begin{equation}
	\hat x^n_\ell = \underset{b}{\amax}\, v_{n, \ell}
\end{equation} 
\begin{equation}
	\hat n = \underset{n}{\amax}\, \log w_n + \sum_\ell \log v_{n,\ell}^\top \hat x^n_\ell
\end{equation}
Then $\hat x^\star = \hat x = \hat x^{\hat n}$ is the approximate maximizer. This is exact when the filters are disjoint $\prod_l \langle v_{nl}, v_{n'l}\rangle = 0 \forall n \neq n'$, which is encouraged by the regularizer.

\section{Posterior approximation}\label{app:posterior_approx}

We implement \Cref{eqn:prior,eqn:like} in the probabilistic programming language NumPyro \citep{phan_composable_2019}. To approximate the posterior distribution $\pr(\theta \mid \Dc_{t-1})$, we primarily use the No-U-Turn-Sampler (NUTS) with a diagonal mass matrix, target acceptance probability of 0.8, and maximum tree depth of 10 \citep{phan_composable_2019}.
8 NUTS chains are run in parallel on a single A100 with 2000 warmup samples, drawing a total of 100000 samples. 
We have observed that a large amount of samples is highly beneficial to prevent overfitting of our mixtures, $q_\psi$, especially in the low SNR regime.
If we obtain a Gelman-Rubin statistic (r-hat) greater than 2 we take this as an indication that the posterior is highly focused, preventing NUTS from sampling efficiently.
In this case the posterior approximation switches to stochastic variational inference (SVI), using a mean-field approximation with a Gaussian distribution with diagonal covariance as the variational family.
For SVI we use the AdamW optimizer with a learning rate of $0.001$, momentum of $0.9$, and adaptive gradient clipping at a ratio of $0.1$ \citep{brock_high-performance_2021}. We use a linear warmup of the learning rate over the first 4000 steps, and use a cosine decay throughout the rest of the training, for a total of 200000 steps. We train using the ELBO with 64 particles to reduce variance of the gradient \citep{wingate_automated_2013, ranganath_black_2014}.

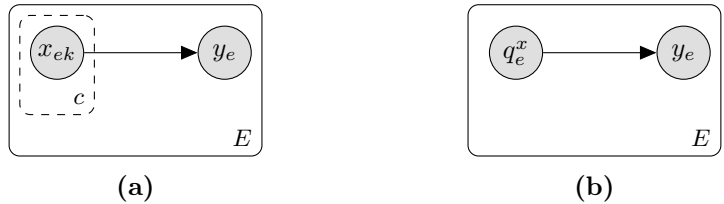
\begin{figure}
\centering
\begin{subfigure}[c]{0.36\textwidth}
    \centering
\begin{tikzpicture}

\node[obs] (y) {$y_{e}$};
  \node[obs, left=1.5cm of y] (a) {$x_{ek}$};

\edge {a} {y} ; 

\plate[dashed] {ap} {(a)} {$c$} ;
  \plate {ay} {(ap)(y)} {$E$} ;

\end{tikzpicture}
\caption{}
\end{subfigure}
\begin{subfigure}[c]{0.36\textwidth}
    \centering
\begin{tikzpicture}

\node[obs] (y) {$y_{e}$};
\node[obs, left=1.5cm of y] (a) {$q^x_e$};

\edge {a} {y} ; 

\plate {ay} {(ap)(y)} {$E$} ;

\end{tikzpicture}
\caption{}

\end{subfigure}
\caption{\textbf{Hierarchical causal model.} (a) Our model can be viewed as a hierarchical causal model, where multiple molecules $x_{ij}$ affect the outcome $y_i$ in each experiment $i$, and there are $c$ total molecules. (b) The collapsed hierarchical causal model, where $q^x$ is the distribution of $x$ (elsewhere we drop the superscript $^x$ for conciseness). This model leads to the estimator in \Cref{eqn:prior,eqn:like}.} \label{fig:hcm}
\end{figure}

\section{Proof of \Cref{ex:O1-search}}\label{app:O1-proof}

\begin{example}[No noise, discrete map]
    Assume activity is discrete, s.t. $\f: \Xc \to \Yc$ for $\Yc=\{0, \ldots |\mathcal{Y}|-1\}$. Assume the model is well specified, $\f \in \{\f_\theta : \theta \in \Theta\}$, and $\pi(\f_\theta = \f) > 0$.
    Assume we can synthesize arbitrary mixtures $q \in \Qc = \Pc(\Xc)$, and our assay has no noise, $y_t = \langle \f, q\rangle$. 

    Then, by optimizing the EIG (\Cref{eqn:info-gain}), we find $x^\star$ after running one experiment.
\end{example}
\begin{proof}
    Number the molecules, so $\Xc = \{1, \ldots, d\}$. Set $q(x) = \frac{1}{Z}\frac{1}{|\Yc|^x}$ where $Z = \sum_{x=1}^d \frac{1}{|\Yc|^x}$ is the normalizing constant.
    Now, $y = \langle \f, q \rangle = \frac{1}{Z}\sum_{x=1}^d \frac{\f(x)}{|\Yc|^j} \Rightarrow Z y = \sum_{x=1}^d \frac{\f(x)}{|\Yc|^j}$. So $\f(x)$ is encoded in the digits of $Z y$, e.g. for $|\Yc| = 10$ we have the decimal expansion $Z y = 0.\f(1)\f(2)\f(3)\ldots$. 
So for this experiment, the posterior over $x^\star$ is a point mass at the true $x^\star$, and the EIG is maximized.
    So the posterior entropy can be driven to zero. Optimizing the EIG will find this solution or anything equally as good, which also drives the entropy to zero.
\end{proof}

\section{Speedup for noiseless 1-sparse molecule-activity maps}\label{app:1-sparse_example}

\begin{example}[No noise, continuous output, 1-sparse]
    Assume the molecule activity map is 1-sparse: $\f(x) = w^\star \Ib(x = x^\star)$. Assume our model is well-specified, so $\f_{\theta}(x) = w \Ib(x = x_0)$ for $\theta = (w, x_0) \in \Rb_+ \times \Xc$ and the prior has positive support $\pi(w^\star, x^\star) > 0$. Assume we can synthesize arbitrary mixtures $q \in \Qc = \Pc(\Xc)$, and our assay has no noise, $y_t = \langle \f, q\rangle$. 

    Then, by optimizing the EIG (\Cref{eqn:info-gain}), we find $x^\star$ after running two experiments in parallel.
\end{example}
\begin{proof}
    Set $q_1(x) = \frac{1}{d} \vec{1}_d$, where $\vec{1}_d$ is the $d$-vector of all ones, and set $q_2(x)$ so that no two entries are the same, $q_2(x) \neq q_2(x')$ for any $x \neq x'$. 
    Now, $y_1 = \langle \f, q_1 \rangle = \frac{1}{d} w^\star \Rightarrow w^\star = d y_1$, so $w^\star$ is identified. Moreover, $y_2 = \langle \f, q_2 \rangle = w^\star q_2(x^\star) \Rightarrow q_2(x^\star) = y_2 / w^\star$, so $q_2(x^\star)$ is identified. Since each $q_2(x): x \in \Xc$ is different, $x^\star$ is identified. 
    So for this choice of experiments, the posterior over $x^\star$ is a point mass at the true $x^\star$, and the EIG is maximized.
\end{proof}
\noindent By contrast, if we tested one molecule at a time, we would need $\Oc(d)$ experiments instead of $\Oc(1)$, by the argument in \Cref{thm:indiv-limit}.

\section{Proof of \Cref{ex:2E-ary_search}}\label{app:2E-ary_search-proof}

\begin{proof}
    The EIG executes $2^E$-ary search. After testing $q_1$, the posterior over $x^\star$ is a uniform distribution over the set 
    \begin{equation}
        X_{y_1} \triangleq \begin{cases}
            \{x : q_1(x) > 0\} \text{ if } y_1 = 1\\
            \{x : q_1(x) = 0\} \text{ otherwise},
        \end{cases}
    \end{equation}
    where the probability of $y_1=1$ is $|X_{1}|/d$.
    Testing a batch $q_{1:E}$ and receiving data $y_{1:E}$ we obtain the posterior $\pr(x^\star \mid q_{1:E}, y_{1:E}) = \mathrm{Uniform}(X_{y_{1:E}})$ where $X_{y_{1:E}} \triangleq X_{ y_1} \cap \ldots \cap X_{y_e}$. The probability of $y_{1:E}$ is $|X_{y_{1:E}}|/d$.
    The $X_{y_{1:E}}$ are non-overlapping and partition $\Xc$.
    So the EIG will be maximized if we set $q_{1:E}$ such that $\lfloor \frac{d}{2^E} \rfloor \leq |X_{y_{1:E}}| \leq \lceil \frac{d}{2^E} \rceil$ for all $y \in \{0, 1\}^E$, making all partitions (as close as possible to) equal size. 

    After testing the first batch, we know $x^\star \in X_{y_{1:E}}$, and our posterior is uniform over this domain. Applying the same logic, the next round will cut the number of possible optima down to $\lceil \frac{d}{(2^E)^2} \rceil$, in the worst case.
    Iterating $k$ times, we are left with just one molecule when $1 = \lceil\frac{d}{(2^E)^k} \rceil \Rightarrow k = \lceil \frac{\log_2 d}{E} \rceil$. \end{proof}

\section{Information gain under a detection threshold} \label{app:proof-info-gain-thresh}

Assays struggle to detect activity when only a very small number of physical molecules in the mixture have activity. Rather than assume we can detect any $\langle \f, q \rangle > 0$, consider a detection threshold $\tau$. 
    Consider an assay $y = \Ib[\langle \f, q \rangle > \tau]$. The maximum expected information gain from a single experiment is \begin{equation} \label{eqn:info-gain-thresh}
    \max_{q \in \Qc} \Hc(\mathrm{Bern}(\Prob_\theta(\langle \f_\theta, q \rangle > \tau))),
    \end{equation} where $\Hc(\mathrm{Bern}(r))$ is the entropy of a Bernoulli distribution with parameter $r$.

The Bernoulli entropy is maximized at $\mathrm{Bern}(0.5)$.
Higher detection thresholds restrict our maximum information gain, making it harder and harder for $\Pr_\theta(\langle \f_\theta, q \rangle > \tau)$ to reach $0.5$. 
For example, assume the maximum activity is $\|\f_\theta\|_\infty$ for all $\theta$.
In the extreme where $\tau = \|\f_\theta\|_\infty$, we only have $\Pr_\theta(\langle \f_\theta, q \rangle > \tau) > 0$ if $q(x)$ is a delta mass at a single sequence, i.e. we revert to standard testing of individual molecules.
In sum, we can only take advantage of mixtures if we can achieve high enough yield that more than one molecular species in the mixture passes the detection threshold.

\begin{proof}
    Define the experimental outcome as a random variable by
    \begin{align}
        Y(q) &= \Ib[\langle \f_\theta, q \rangle > \tau]\\
        &= \begin{cases}
            1, &\mathrm{if} \quad \langle\f_\theta, q \rangle > \tau \\
            0, &\mathrm{otherwise},
        \end{cases}
    \end{align}
    where the randomness stems from the randomness in $\theta$. This is a Bernoulli random variable with parameter $\Pr_\theta(\langle \f_\theta, q \rangle > \tau)$.
    The EIG is given by the mutual information between $Y(q)$ and $\theta$, using $\EIG_\theta (q) = \Hc(Y(q)) - \Hc(Y(q) \mid \theta)$.
    Since $Y(q)$ is a deterministic function of $\theta$, the conditional entropy is zero, $\Hc(Y(q) \mid \theta) = 0$. 
    Therefore, the $\EIG$ is 
    \begin{align}
        \EIG_\theta(q) &= \Hc(Y(q))\\
        &= \Hc(\mathrm{Bern}(\Pr_\theta(\langle \f_\theta, q \rangle > \tau)))
    \end{align}
\end{proof}

\section{Proof of \Cref{prop:gaussian_mi_bound}} \label{app:proof-EIGbound}

\begin{proof}The signal of our experiment is $S \triangleq \langle \f_\theta, q \rangle$. This is a random variable with randomness stemming from $\pi$, the prior on $\theta$.

    By the chain rule $\MI(\theta, S; Y) = \MI(S;Y) + \MI(\theta;Y \mid S)$. Since $\pr(y \mid \theta, S) = \pr(y \mid S)$, then $\MI(\theta;Y \mid S) = 0$. So for a fixed design $q$, the mutual information between $Y$ and $S$ is equal to the $\EIG$, $\MI(\theta;Y) = \MI(S;Y) = \EIG_\theta(q)$. 
    We now bound $\MI(S;Y)$.
    The first bound stems from the bound on the entropy of $S$, by bounding the mutual information by the prior entropy as 
    \begin{align}
        \MI(S;Y) &= \Hc(S) - \Hc(S \mid Y) \\
        &\leq \Hc(S) = \Hc(\langle \f_\theta, \bar{q} \rangle),
    \end{align}
    where we use the fact that $S$ comes from a finite set $\Zc$. This is an upper bound on the amount of information we can encode into an experiment. 
    The second bound stems from the channel capacity of an additive gaussian noise channel derived by Shannon \citep{shannon_mathematical_1948} by
    \begin{equation}\label{eq:mi(sy)}
        \MI(S;Y) = h(Y) - h(Y \mid S) = h(Y) - \frac{1}{2}\log (2\pi e \sigma^2).
    \end{equation}
    From our assumptions on $\langle \f_\theta, q\rangle$, $Y$ must have bounded second moment. So, the differential entropy, $h(Y)$, can be bounded by its maximum entropy, which is achieved when $Y$ is Gaussian distributed with variance $\mathrm{Var}(Y) = \mathrm{Var}(S + \epsilon\sigma) = \mathrm{Var}(S) + \sigma^2$
    \begin{equation}\label{eq:h(y)}
        h(Y) \leq \frac{1}{2}\log(2\pi e(\mathrm{Var}(S) + \sigma^2))
    \end{equation}
    Combining \Cref{eq:mi(sy)} and \Cref{eq:h(y)} we get
    \begin{align}
        \MI(S;Y) &\leq \frac{1}{2}\log(2\pi e(\mathrm{Var}(S) + \sigma^2)) - \frac{1}{2}\log (2\pi e \sigma^2) \\
        \MI(S;Y) &\leq \frac{1}{2}\log\left(1 + \frac{\mathrm{Var}(\langle \f_\theta, q \rangle)}{\sigma^2}\right).
    \end{align}
    Defining $\textsc{SNR}(q) \triangleq \frac{c^2}{\sigma^2}\mathrm{Var}(\langle \f_\theta, \bar{q} \rangle)$, and combining the two bounds we get
    \begin{align}
        \EIG_\theta(q) \leq \min\left\{ \Hc(\langle \f_\theta, \bar{q} \rangle), \, \frac{1}{2}\log \left(1 +  \frac{c^2}{\sigma^2}\mathrm{Var}(\langle \f_\theta, \bar{q} \rangle)\right) \right\}.
    \end{align}
\end{proof}

\subsection{$\textsc{SNR}(q)$ is maximized at $\bar q = \delta_{x}$} \label{app:proof-EIGbound2}
\begin{proof}
To maximize $\textsc{SNR}(q) = \frac{c^2}{\sigma^2}\mathrm{Var}_\pi(\langle\f_\theta,\bar q\rangle)$, we need to maximize $\mathrm{Var}_\pi(\langle\f_\theta,\bar q\rangle)$.
Define $\tilde\f_\theta\triangleq\f_\theta-\Eb_\pi[\f_\theta]$, the centered $\f_\theta$. Then $\mathrm{Var}_\pi(\langle\f_\theta,\bar q\rangle) = \Eb_\pi[(\langle\tilde\f_\theta,\bar q\rangle)^2]$. For any fixed $\theta$, $\langle\tilde\f_\theta,\bar q\rangle$ is a linear function in $\bar q$. Since $x^2$ is a convex function, $(\langle\tilde\f_\theta,\bar q\rangle)^2$ is a convex function for any fixed $\theta$.

So, the variance is an expectation over pointwise convex functions, which means that it is a convex function. We are therefore trying to find the maximum of a convex function $\mathrm{Var}_\pi(\langle\f_\theta,\bar q\rangle)$ over a compact convex set $\Pc(\Xc)$. By Bauer's maximum principle, this attains its maximum at an extreme point \citep{bauer_minimalstellen_1958}. For $\Pc(\Xc)$, the extreme points are $\{\delta_x:x \in \Xc\}$, the set of measures concentrated on a single molecule. We can therefore establish 
\begin{equation}
    \exists\delta_{x} \in \underset{\bar q \in \Pc(\Xc)}{\amax} \mathrm{Var}(\langle\f_\theta,\bar q\rangle)
\end{equation}
\end{proof}

\section{Proof of \Cref{ex:gaussian_example}} \label{app:proof-gaussian_example}

Assume the molecule activity map is 1-sparse with maximum value one, $\f_\theta(x) \in \{ \Ib(x = \theta):\theta\in\Xc \}$, and the prior is uniform, $\pi(\theta) = 1/d$.
Assume a synthesis model that creates uniform mixtures, $q(x) = \frac{1}{d_q} \Ib(x \in U)$ for $U \subseteq \Xc$, $\|q\|_1=c=1$ and $d_q \triangleq |U|$. 
Assume Gaussian noise $y = \langle \f_\theta, q\rangle + \sigma \epsilon$ where $\epsilon \sim \mathrm{Normal}(0, 1)$.

\paragraph{EIG bound from \Cref{prop:gaussian_mi_bound}}

\begin{equation}
    \EIG_{\theta}(q) \leq \min\left\{ -p_q\log p_q - (1-p_q)\log(1-p_q), \, \frac{1}{2}\log \left(1 + \frac{1}{\sigma^2}\left(\frac{1}{d_qd} - \frac{1}{d^2}\right)\right) \right\},
\end{equation}
where $p_q \triangleq \frac{d_q}{d}$. The first term is strictly increasing on the interval $d_q \in [0,\frac{d}{2}]$, while the second term is strictly decreasing in $d_q$.

The exact EIG is given by,
\begin{align}
    \EIG_{\theta}(q) = \left(1-\frac{d_q}{d}\right)\KL\left(\Nc\left(0,\sigma^2\right)||p(y)\right) + \frac{d_q}{d}\KL\left(\Nc\left(\frac{1}{d_q},\sigma^2\right)||p(y)\right)
\end{align}
where $p(y) = (1-\frac{d_q}{d}) \cdot \Nc (y;0,\sigma^2) + \frac{d_q}{d}\cdot\Nc(y;1/d_q,\sigma^2)$, with $\Nc(y;\mu,\sigma^2)$ the Gaussian pdf. We approximate this by Monte Carlo to obtain \Cref{fig:MI_vs_dq}.

\begin{proof}
Define the random variable $S(q) = \langle \f_\theta, q \rangle$, where $\theta \sim \pi(\theta)$. It has probability mass function (pmf)
    \begin{equation}
        \pr(S(q)) =
        \begin{cases}
            \frac{d_q}{d}, &\mathrm{if} \,S(q) = \frac{1}{d_q} \\
            1 - \frac{d_q}{d}, &\mathrm{if} \, S(q) = 0
        \end{cases}
    \end{equation}
    This is the pmf of a Bernoulli random variable, scaled by $\frac{1}{d_q}$, which has entropy $\Hc(\pr(S(q))) = -p_q\log p_q - (1-p_q)\log(1-p_q)$, and variance 
    \begin{align}
        \mathrm{Var}_\pi[S(q)] &= \frac{1}{d_q^2} \cdot \frac{d_q}{d}(1-\frac{d_q}{d}) \\
        &=\frac{1}{d_qd} - \frac{1}{d^2}.
    \end{align}
    Plugging into \Cref{prop:gaussian_mi_bound} we get 
    \begin{equation}
        \EIG_{\theta}(q) \leq \min\left\{ -p_q\log p_q - (1-p_q)\log(1-p_q), \, \frac{1}{2}\log \left(1 + \frac{1}{\sigma^2}\left(\frac{1}{d_q d} - \frac{1}{d^2}\right)\right) \right\}.
    \end{equation}
\end{proof}
For this example, we can even compute the exact EIG.
\paragraph{True MI}
\begin{proof}
    The $\EIG$ can be defined as the expected KL-divergence between the prior and the posterior by
    \begin{equation}
        \EIG_\theta(q) = \EIG_S(q) = \Eb_{S(q)}\KL(\pr(Y(q) \mid S(q)) \| \pr(Y(q))),
    \end{equation}
    where $Y(q) = S(q) + \sigma\epsilon$. This is a scaled Bernoulli random variable plus a Gaussian random variable, with pdf $p(y) = (1-\frac{d_q}{d}) \cdot \Nc (y;0,\sigma^2) + \frac{d_q}{d}\cdot\Nc(y;1/d_q,\sigma^2)$, where $\Nc(y;\mu,\sigma^2)$ is the Gaussian pdf. Since $\pr(S(q))$ is a Bernoulli distribution pmf, the $\EIG$ can be simplified to
    \begin{align}
        \EIG_S(q) &= \left( 1-\frac{d_q}{d} \right)\KL(\pr(Y(q) \mid S(q)=0) \| \pr(Y(q))) \\
        &+ \left( \frac{d_q}{d} \right)\KL(\pr(Y(q) \mid S(q)=\frac{1}{d_q}) \| \pr(Y(q))) \\
        &=  \left( 1-\frac{d_q}{d} \right)\KL(\Nc\left(0,\sigma^2\right)) \| \pr(y)) + \left( \frac{d_q}{d} \right)\KL(\Nc\left(\frac{1}{d_q},\sigma^2\right)) \| \pr(y))
    \end{align}
\end{proof}

\section{Proof of \Cref{prop:paninski}} \label{app:proof-paninski}
\begin{proof}
    Let $\Xc=[d]$ be the search space of size $|\Xc| = d$. Let $\f_\theta : \Xc \to \Rb_+$ be an arbitrary molecule-activity map.
    We aim to gain information about $\f_\theta$, not its parameters.
    
    $q \in \Mc^+_c(\Xc)$ is a positive measure subject to a 1-norm constraint $\|q\|_1=c$. Additionally, define the set of delta functions as $\Sc_c(\Xc) = \{c\delta_x :x\in \Xc\} \subset \Mc^+_c(\Xc)$ subject to the same constraint. Assume Gaussian noise $\pr(y \mid \f_\theta, q, \sigma) = \Nc(\langle \f_\theta, q \rangle, \sigma^2)$. The score function for this likelihood is 
    \begin{align}
    \nabla_{\f_\theta}\log\Nc(\langle \f_\theta, q \rangle, \sigma^2) &= \nabla_{\f_\theta}\left(-\frac{1}{2}\log(2\pi\sigma^2) - \frac{(y - \langle \f_\theta, q \rangle)^2}{2\sigma^2}\right)\\
    &= \nabla_{\f_\theta}\left( - \frac{(y - \langle \f_\theta, q \rangle)^2}{2\sigma^2}\right)\\
    &= \frac{y - \langle \f_\theta, q \rangle}{\sigma^2}\nabla_{\f_\theta} \langle \f_\theta, q \rangle\\
    &= \frac{y - \langle \f_\theta, q \rangle}{\sigma^2} q
    \end{align}
    The information matrix is given by the outer product of the score function, evaluated at the true parameter $\f_{\theta_0}$ by 
    \begin{align}
        \Ic_{\f_{\theta_0}}(q) &= \Eb_{\pr(y \mid q, \f_\theta)}\left[\left.\nabla_{\f_\theta}\log\pr(y \mid q, \f_\theta) \nabla_{\f_\theta}\log\pr(y \mid q, \f_\theta)^\top\right|_{\f_\theta=\f_{\theta_0}}\right]\\
        &= \Eb_{\pr(y \mid q, \f_\theta)}\left[\left.\left(\frac{y - \langle \f_\theta, q \rangle}{\sigma^2} q\right)\left(\frac{y - \langle \f_\theta, q \rangle}{\sigma^2} q\right)^\top\right|_{\f_\theta=\f_{\theta_0}}\right]\\ 
        &= \frac{1}{\sigma^4}\Eb_{\pr(y \mid q, \f_\theta)}\left[\left.\left(y - \langle \f_\theta, q \rangle \right)^2 q \otimes q\right|_{\f_\theta=\f_{\theta_0}}\right]\\
        &=  \frac{\Eb_{\pr(y \mid q, \f_\theta)}\left[\left.\left(y - \langle \f_\theta, q \rangle \right)^2\right|_{\f_\theta=\f_{\theta_0}}\right]}{\sigma^4} q \otimes q\\
        &= \frac{1}{\sigma^2} q \otimes q
    \end{align}
    The result from Paninski \citep{paninski_asymptotic_2005} states that for information optimal sampling, the posterior $\pr(\f_\theta \mid \Dc_t)$ is asymptotically normal, with covariance matrix $\frac{\sigma^2_{info}}{t}$ determined by
    \begin{equation}
        \sigma^2_{info}(\Qc) = \left(\underset{C \in co(\Ic_{\f_{\theta_0}}(\Qc))}{\amax} \det(C)\right)^{-1},
    \end{equation}
    where $co(\Ic_{\f_{\theta_0}}(\cdot))$ denotes the convex closure of the set of information matrices, and $\det(A)$ is the determinant of a matrix $A$.
    We will now show how restricting $\Qc$ to $\Sc_c(\Xc)$ does not change the asymptotic variance over the case of the full measure $\Mc^+_c(\Xc)$.
    First, note that $\Ic_{\f_{\theta_0}}(q) = \frac{1}{\sigma^2} q \otimes q \in \Rb_+^{d \times d}$ is positive semidefinite by construction, and that it has a fixed sum $\sum_{i,j}\Ic_{\f_{\theta_0}}(q)_{i,j} = \frac{1}{\sigma^2} q_i q_j = \frac{c^2}{\sigma^2}$. 
    The convex closure of any set of information matrices inherits these properties, because the set of positive semidefinite matrices is convex, as is the set of matrices with a fixed sum, and the set of positive-valued matrices. 
    We can write this succinctly as $co(\Ic_{\f_{\theta_0}}(\Qc)) \subseteq \{M : \ M_{ij}\geq 0\ \forall i,j,\ \lambda_1,...\lambda_d \geq 0,\ \sum_{i,j}M_{i,j}=\frac{c^2}{\sigma^2}\}$, where $\lambda$ denotes an eigenvalue.
    We can use $\det(C) = \prod_{i=1}^d\lambda_i$, as well as $\mathrm{Tr}(C) =\sum_{i=1}^d \lambda_i= \sum_{i=1}^d C_{i,i} \leq \sum_{i,j}C_{i,j}=\frac{c^2}{\sigma^2}$ to conclude that the determinant is bounded by $\det(C) \leq \prod_{i=1}^d \frac{c^2}{d\sigma^2}$, as this is the point where all eigenvalues are equal, and the trace is maximal. Furthermore, since
    \begin{align}
        \det(\mathrm{I}_d\cdot \frac{c^2}{d\sigma^2}) &= \prod_{i=1}^d \frac{c^2}{d\sigma^2}, \\
        &\mathrm{and} \nonumber\\
        \mathrm{I}_d\cdot \frac{c^2}{d\sigma^2} &\in co(\Ic_{\f_{\theta_0}}(\Sc_c(\Xc))),
    \end{align}
    this supremum is attained within $\Sc_c(\Xc)$.
    To say it another way, the determinant, and hence the asymptotic variance, is maximized at $C^\star = \mathrm{I}_d\cdot \frac{c^2}{d\sigma^2}$. Because $C^\star \in co(\Ic_{\f_{\theta_0}}(\Sc_c(\Xc)))$, there can be nothing to gain in the asymptotic sense by lifting to $q \in \Mc_c^+(\Xc)$, and we may conclude
    \begin{equation}
        \sigma^2_{info}(\Sc_c(\Xc)) = \sigma^2_{info}(\Mc_c^+(\Xc))
    \end{equation}
\end{proof}

\section{Details on sparse synthetic oracle simulations}\label{app:fsp}

Our sparse oracle is created by sparsifying draws from our prior $\pi(\theta)$ and setting the maximum to one $\|\f\|_\infty=1$. 
The prior is described in \Cref{app:bnn-prior}.
We set $N=1$ and $w=1$, and then draw 10 independent repeats $\theta_{1:10} \sim \pi(\theta)$.
Afterwards, we construct the oracle by $\f(x) = \f_{\theta_j}(x) \cdot \Ib[ d_H(x,x^\star_{\theta_j}) \leq 1]$, where $d_H$ denotes the Hamming distance. 

Such molecule-activity maps $\f: \Xc \to [0,1]$ have the following qualities:
\begin{enumerate}
    \item Exactly one DNA sequence $x^\star$ has activity $\f(x^\star)=1$.
    \item Every single mutation has lower activity. $0 < \f(x) < \f(x^\star) \, \forall \, d_H(x,x^\star)=1$.
    \item Every double mutation abolishes activity. $\f(x)=0 \forall d_H(x,x^\star)>1$.
\end{enumerate}

For both LIDS and Thompson, we use the same hyperparameters of the model $\f_\theta$. 
We set the number of filters to $N=8$, regularization strength $\lambda=1$, and use $\pr(w)=\delta_{w=1}$ as the prior on $\|\f\|_\infty$.
Hyperparameters for EIG optimization and posterior approximation are described in \Cref{app:variational_dist_training,app:posterior_approx}, respectively.

\section{CI calculation}\label{app:CI}

95\% confidence intervals are calculated using a critical t-value and the delta method.
Let $\hat\f_n \triangleq \f(\hat x_n^\star)$ be the activity of the test molecules for the $N$ replicates, and $\f(x^\star)$ the known true maximum activity.
First, compute the standard error and mean.
\begin{align}
    \hat\mu &= \frac{1}{N}\sum_{n=1}^N\hat\f_n\\
    s^2 &= \frac{1}{N-1} \sum_{n=1}^N (\hat\f_n - \hat\mu)^2\\
    SE &= \frac{s}{\sqrt{N}}
\end{align}
This naive standard error (SE) can be very large, making our confidence interval go below 0, or above the maximum possible activity. To make the confidence interval (CI) stay bounded in $[0,\f(x^\star)]$, we use a logit transformation $g: [0,\f(x^\star)] \to \Rb$, along with the delta method.
Specifically, we use $g(y)\triangleq \log(\frac{y}{\f(x^\star) - y})$, which has derivative $g'(y) = \frac{\f(x^\star)}{y(\f(x^\star)-y)}$, and inverse $g^{-1}(y) = \frac{\f(x^\star)}{1 + \exp(-y)}$.
\begin{align}
    SE_{logit} &= g'(\hat\mu)\cdot SE\\
    CI_{logit} &= g(\hat\mu)\pm t_{N-1}(0.975)\cdot SE_{logit}\\
    CI &= g^{-1}(CI_{logit})
\end{align}
$t_{N-1}(0.975)$ is the $0.975$ quantile of the student-t distribution with $N-1$ degrees of freedom.
For the $N=5$ and $N=10$ applied in this article, this corresponds to $99.45\%$ and $97.63\%$ CI respectively, if we were to use the Gaussian distribution instead of a student-t.
We then plot the arithmetic mean $\hat\mu$, along with the $CI \subset [0,\f(x^\star)]$.

\section{Details on PLM oracle simulations}\label{app:fplm}

For the hyperparameters of the surrogate $\f_\theta$, we set the number of filters to $N=20$, regularization strength $\lambda=1$, and use $\pr(w)=1 + \mathrm{Gamma}(1,1)$ as the prior on $\|\f\|_\infty$.
Hyperparameters for training and posterior approximation are described in \Cref{app:variational_dist_training,app:posterior_approx}.

\subsection{Importance sampling $\langle \f, q \rangle$}\label{app:IS}

To simulate an experiment we need to evaluate $\langle \f, q \rangle = \Eb_{x \sim q(x)}[\f(x)]$. 
To do so we exploit the fact that $\pr(x)$ can be sampled from, as it is given by a PLM, which is a generative model.
We set a defensive mixture distribution $s(x) = \frac{1}{2} q(x) + \frac{1}{2}\pr(x)$ as the proposal distribution \citep{hesterberg_weighted_1995} and approximate the expectation by importance sampling.
\begin{equation}
    \langle \f, q \rangle \approx \frac{1}{N}\sum_{i=1}^N \frac{q(x^i)}{s(x^i)}\f(x^i),\quad x^i \sim s(x)
\end{equation}

\newpage
\section{Ablation of posterior approximation}
\begin{figure}[h]
    \vspace{-1em}
    \centering
    \includegraphics[width=0.45\linewidth]{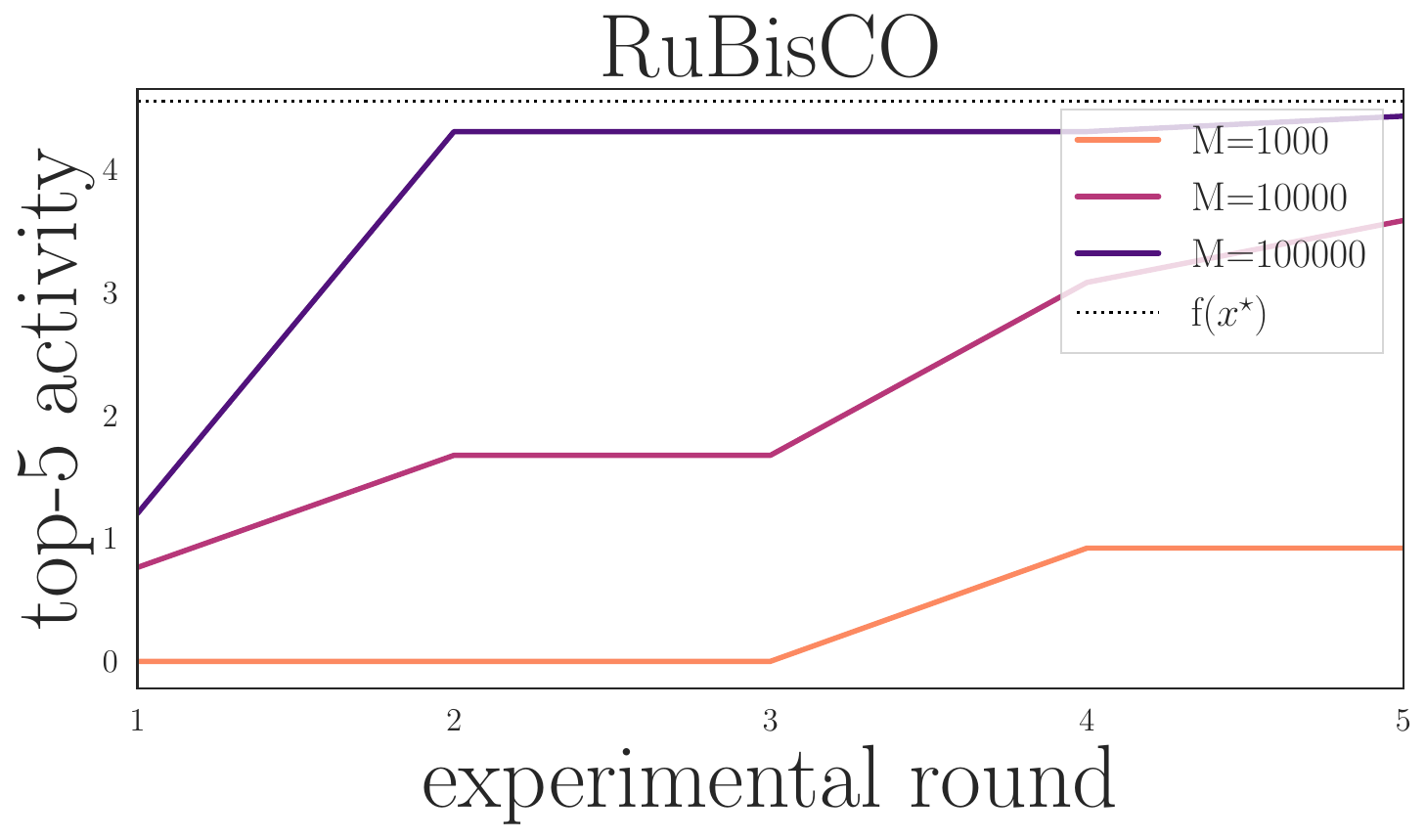}
    \caption{\textbf{Posterior approximation.} LIDS using different number of posterior samples $M$. The setup is the same as in \Cref{sec:res_plm}.}
    \label{fig:M-ablation}
    \vspace{-2em}
\end{figure}

\end{document}